\documentclass[letterpaper, 10 pt, conference]{ieeeconf}

\IEEEoverridecommandlockouts   

\usepackage[utf8]{inputenc} 
\usepackage[T1]{fontenc}    
\usepackage{url}            
\usepackage{booktabs}       
\usepackage{amsfonts}       
\usepackage{graphicx}%
\usepackage{amsmath}
\usepackage{amssymb}
\usepackage{xcolor}
\usepackage{accents}
\usepackage{caption}
\usepackage{cite}
\usepackage{xcolor}
\usepackage{multirow}
\usepackage{siunitx}
\usepackage{amsmath,bm}
\usepackage{algorithm}
\usepackage{algpseudocode}
\usepackage{caption}
\usepackage{subcaption}
\usepackage{stfloats}
\usepackage{soul}
\setstcolor{red}
\usepackage[colorlinks=true, urlcolor=red, breaklinks=true]{hyperref}
\usepackage{breakurl}

\newtheorem{theorem}{Theorem}

\newcommand{\longthmtitle}[1]{\mbox{} \emph{(#1):}}

\renewcommand{\baselinestretch}{.98}

\title{
CVIRO: A Consistent and Tightly-Coupled Visual-Inertial-Ranging Odometry on Lie Groups
}
\author{Yizhi Zhou, Ziwei Kang, Jiawei Xia, Xuan Wang 
\thanks{Y. Zhou, Z. Kang, J. Xia, and X. Wang are with the Department of Electrical and Computer Engineering, George Mason University, Virginia, USA. This research is supported by ECCS-$2332210$.}
}

\begin{document}
\maketitle
\thispagestyle{empty}
\pagestyle{empty}

\begin{abstract}
Ultra-Wideband (UWB) is widely used to mitigate drift in visual-inertial odometry (VIO) systems. Consistency is crucial for ensuring the estimation accuracy of a UWB-aided VIO system. 
An inconsistent estimator can degrade localization performance, where the inconsistency primarily arises from two main factors: (1) the estimator fails to preserve the correct system observability, and (2) UWB anchor positions are assumed to be known, leading to improper neglect of calibration uncertainty.
In this paper, we propose a consistent and tightly-coupled visual-inertial-ranging odometry (CVIRO) system based on the Lie group. Our method incorporates the UWB anchor state into the system state, explicitly accounting for UWB calibration uncertainty and enabling the joint and consistent estimation of both robot and anchor states. Furthermore, observability consistency is ensured by leveraging the invariant error properties of the Lie group. We analytically prove that the CVIRO algorithm naturally maintains the system’s correct unobservable subspace, thereby preserving estimation consistency. Extensive simulations and experiments demonstrate that CVIRO achieves superior localization accuracy and consistency compared to existing methods.
\end{abstract}

\section{INTRODUCTION}
Visual-inertial odometry (VIO) is widely used in autonomous driving and unmanned aerial vehicles due to its accuracy, reliability, and lightweight design \cite{OVS, VM}. However, VIO inevitably suffers from cumulative drift for long-term running due to the lack of global information during estimation. To mitigate this issue, VIO systems commonly incorporate global measurements from additional sensors, such as GPS. However, GPS relies heavily on open environments, rendering it ineffective in indoor or GPS-denied settings. Recent studies have explored integrating Ultra Wideband (UWB) measurements into VIO systems. UWB utilizes ranging measurements between robots and stationary anchors to enhance localization performance, which are suitable for both indoor and outdoor GPS-denied scenarios \cite{Jia2022, Hc2024, NTM2021}.


In general, a well-designed UWB-aided VIO system should not only be error-bounded but also consistent while fusing ranging measurements.  Inconsistent estimators often underestimate system uncertainty, ultimately reducing estimation accuracy. A primary cause of this inconsistency is improper handling of uncertainties associated with the fused information. For example, the commonly adopted loosely-coupled method, which separately estimates anchor positions and robot localization, requires precise calibration of UWB anchors before integration. However, treating anchor positions as fixed priors without properly considering their calibration uncertainty ultimately compromises system consistency. Moreover, many loosely-coupled methods require manual computation of anchor positions, which is labor-intensive and prone to errors. In contrast, the tightly-coupled method provides an ideal solution by treating anchor positions as unknown variables and directly incorporating the anchor position into the system state, thereby ensuring consistency throughout the estimation process. 

In addition to inaccuracies in UWB calibration, another major source of inconsistency is the failure to preserve the system’s original observability properties. An estimator lacking observability consistency can introduce spurious information into unobservable directions, leading to overconfidence and degraded accuracy \cite{HG2014, FEJ}. A commonly used approach to address this issue involves artificially modifying the estimator’s observability subspace, either by changing the linearization point of measurement Jacobians, known as the First-Estimate Jacobian (FEJ) method \cite{FEJ}, or by directly adjusting the Jacobian itself, referred to as the Observability-Constrained (OC) method \cite{HG2014}. However, the OC method often incurs high computational costs, while the FEJ method suffers from poor initialization \cite{FEJ2}. In UWB-aided VIO systems, the performance of the FEJ method is further degraded due to inherently poor UWB initialization quality and imperfect calibration results.




In this paper, we consider a more general scenario in which anchor positions are unknown and subject to uncertainties arising from calibration. Furthermore, UWB initialization and calibration is imperfect, introducing additional sources of estimation error. 
To effectively and consistently fuse UWB measurements under these two conditions, we propose a consistent and tightly-coupled visual-inertial-ranging odometry (CVIRO) formulated on the Lie group.
First, to enable autonomous calibration and explicitly model calibration uncertainty during estimation, we incorporate UWB anchor positions into the system state, allowing for the joint estimation of both robot and anchor states in a tightly-coupled manner. Second, we propose a Lie group-based multi-constrained Kalman filter (MSCKF) to CVIRO, where the system state is defined on the Lie group and leverages the properties of right-invariant errors. We proved that the CVIRO estimator preserves four unobservable directions, which are exactly aligned with the system's original unobservable directions, thereby ensuring observability consistency. Unlike FEJ or OC methods, which require external manipulation of the system's linearization jacobian to maintain observability consistency  that may introduce addition errors \cite{FEJ2, FEJ}, our approach naturally preserves the correct observability properties without any artificial remedies. In summary, the main contribution of this work can be summarized as
\begin{itemize}
    \item We propose a consistent and tightly-coupled visual-inertial-ranging odometry (CVIRO) system based on the Lie group for joint robot localization and UWB sensor calibration.
    \item We analytically derive that the proposed CVIRO system has four unobservable directions and naturally preserves the correct observability properties of the original system, ensuring estimation consistency (see the supplementary material for more details, available at \cite{zzSPLY2025} or \textit{the provided link}\footnote{\scriptsize\url{https://mason.gmu.edu/~xwang64/papers/cviro_supp.pdf}}).
    \item We perform extensive simulations and experiments to validate the effectiveness of the proposed system as well as the derived observability properties that guarantees the system's consistency.
\end{itemize}

\section{Related Works}
UWB-aided VIO systems have been extensively studied to improve VIO estimation performance and enable drift-free lifelong localization by leveraging UWB sensor measurements. Several earlier works have employed UWB measurements in a loosely coupled manner, requiring UWB anchors to be calibrated prior to use \cite{WC2017, PF2017, ZJ2022, YB2021, SSL2021}, and have demonstrated promising localization performance.
Nevertheless, these methods require offline calibration of anchor positions, which can be challenging in complex and large-scale environments. Furthermore, all these methods assume that UWB anchor calibration is precise and neglect the uncertainty associated with calibration errors. This oversight can lead the estimator to underestimate the system’s uncertainty during localization, ultimately resulting in inconsistency. 

In recent years, researchers have proposed tightly-coupled methods that directly incorporate anchor states into the system state, enabling the joint estimation of both anchor positions and the robot pose \cite{ZS20222, TJ2018, NT2020}. This approach offers two key advantages: (1) UWB anchors can be automatically calibrated using a mobile robot, eliminating the need for offline calibration, and (2) calibration uncertainty is explicitly accounted for within the estimator, rather than being ignored as in loosely-coupled methods. As a result, this significantly enhances the system's consistency, leading to improved localization accuracy. A visual-inertial range SLAM system was proposed in \cite{Cao2021}, which externally incorporates UWB measurements into a factor graph to jointly estimate UWB anchor positions and robot states. To mitigate the non-line-of-sight (NLOS) issue of UWB sensors, researchers in \cite{Hc2024} introduced an object detection-based calibration method that leverages visual measurements instead of UWB ranging data to initialize multiple anchors with unknown positions.
However, these methods are built upon optimization-based approaches to achieve high estimation accuracy, which are generally computationally intensive, making them less suitable for real-time applications and resource-constrained environments.
In \cite{DGS2023}, a UWB-aided VIO system is proposed that not only estimates anchor positions but also jointly calibrates the bias terms in the UWB measurement model to further enhance localization accuracy. 

All the above-mentioned filter-based methods \cite{DGS2023} achieve accurate and lightweight estimation performance, while ensuring consistency by explicitly incorporating UWB calibration uncertainty. Nevertheless, similar to most VIO systems, another critical source of inconsistency arises from the system’s observability properties \cite{FEJ}. Observability can determine the minimal measurements
required for the reconstruction of system states, which is crucial consistent state estimation. If an estimator fails to preserve the correct system's observability properties, the estimator may introduce additional errors into the unobservable subspaces and ignore the uncertainty from these subspaces that makes the estimator overconfident and inconsistent. Specifically, Huang et al. \cite{HG2014} proved that a VIO system generally has four unobservable subspaces corresponding to global translations and yaw orientations, where the yaw orientation subspace is mistakenly eliminated during EKF estimation, resulting in inconsistency in the estimation process. One approach to mitigate this issue is the first-estimate Jacobian (FEJ) technique, which preserves the correct observability properties of the EKF system by directly modifying the linearization point of EKF. This technique has been further extended to filter-based UWB-aided VIO systems \cite{Jia2022} for consistent estimation. However, FEJ-based methods are sensitive to the initialization performance, as poor system initialization can introduce additional estimation errors, even when observability consistency is maintained \cite{FEJ2}. Instead of directly modifying the linearization jacobians like FEJ does, another line to maintain observability consistency is to employ the invariant observer theory which used the Lie group $SE(n)$ to represent the robot state \cite{YYL2022, Xuz2023}. 
By leveraging the properties of invariant errors, such estimators are referred to as Invariant Extended Kalman Filters (InEKF) \cite{BA2017} and naturally preserve the correct unobservable subspace of the original system, ensuring consistency. InEKF has been widely applied in robot navigation \cite{Heo2018}, state estimation \cite{BM2018, RH2020}, and SLAM \cite{YYL2022, zhang2017convergence} to design estimators with observability consistency. 
In this paper, we apply a Lie group-based state representation that explicitly incorporates UWB anchor positions into the system state. We prove that the proposed approach preserves the correct observability properties of the visual-inertial ranging odometry system, thereby ensuring consistency.

\section{System Design}
In this section, we provide a detailed introduction to the CVIRO system, which integrates measurements from an IMU, camera, and UWB within Lie group-based MSCKF framework. In particular, consider a robot in a 3-D environment equipped with a visual-inertial sensor for ego-motion measurement and a UWB tag to measure the distance with multiple UWB anchors with unknown positions. We first define the IMU and camera frames at time $t_k$ as $I_k$ and $C_k$, respectively, while $G$ represents the global frame. The primary objective of the system is to jointly estimate the robot's (IMU's) states along with the positions of the UWB anchors simultaneously.
Notably, unlike the standard MSCKF framework, which is formulated in the vector space \cite{MSCKF}, the proposed CVIRO system operates directly on the matrix Lie group representation, leveraging the inherent manifold structure to enhance the consistency and accuracy of the estimator.

\subsection{System state}
The state of the proposed CVIRO at time $t_k$ contains the current IMU state, the position of the UWB anchors, and a sliding window of $m$ historical IMU clones as
\begin{align}\label{eq_state}
\mathbf{X}_k&=\left(
\mathbf X_{I_k}, \mathbf X_{C} \right)\nonumber\\
\mathbf X_{C}&=\left(\mathbf C_{k-1},  \cdots, \mathbf C_{k-m} \right)
\end{align}
Here, the IMU state and the UWB anchor state are jointly represented as $\mathbf{X}_{I_k}=(\mathbf T_{k}, \mathbf B_{k})\in SE_{2+L}(3)\times \mathbb R^6$, where $SE_{2+L}(3)$ represents the augmented Special Euclidean Group in three dimension \cite{BA2017}, and $L$ represents the number of UWB anchors. For simplicity, we assume that the state only contains one UWB anchors, i.e., $L=1$, which can be easily generalizable to multi-anchor cases based on the definitions of $SE_{2+L}(3)$ group. Specifically, the state representation $\mathbf T_k\in SE_{2+L}(3)$ is defined as
\begin{align}
\mathbf T_{k}&=\left[
\begin{array}{c|ccc}
{^{G}_{I_{k}}\mathbf{R}} & ^G\mathbf{v}_{I_{k}} & ^G\mathbf{p}_{I_{k}} & ^G\mathbf{p}_{u}\\
\hline
\mathbf 0_{3}&\multicolumn{3}{c}{\mathbf I_{3}}
\end{array}
\right]\in\ SE_3(3), 
\end{align}
that contains the IMU's rotation ${^G_{I_{k}}\mathbf{R}}\in SO(3)$ from the inertial frame ${I_{k}}$ to global frame ${G}$ at time $t_k$, and IMU's position ${^G{\mathbf p}_{I_{k}}}$ and velocity ${^G{\mathbf v}_{I_{k}}}$ in the global frame. It is important to highlight that, in addition to the IMU states, $\mathbf T_{k}$ also contains the UWB anchors' positions $^G \mathbf p_u$ since it is also unknown parameters that need to be estimated. 
$\mathbf B_k=\begin{bmatrix} \mathbf b_{\omega_k}^\top&\mathbf b_{a_k}^\top \end{bmatrix}^\top\in\mathbb R^6$ consists the IMU's gyroscope bias $\mathbf b_{\omega_k}$ and accelerometer bias 
$\mathbf b_{a_k}$.
For $i\in\{k-m, \cdots, k-1\}$, 
$\mathbf{C}_i\in SE(3)$,  denotes the historical clones of the IMU’s pose at time $t_i$ given by
\begin{align}
\mathbf{C}_i=\begin{bmatrix}
{^G_{I_i}\mathbf R} & ^G\mathbf p_{I_i}\\
\mathbf 0_{1\times 3} & 1
\end{bmatrix}    
\end{align}
By defining $\hat {\mathbf X}_k=\left(\hat {\mathbf X}_{I_k}, \hat{\mathbf X}_{C}\right)$ as the state estimate of $\mathbf X_k$, we formulate the error state $\widetilde{\mathbf X}_k$ as
\begin{align}
\widetilde{\mathbf X}_k=\left(\hat{\mathbf T}_k \mathbf T_k^{-1}, \hat{\mathbf B}_k-\mathbf B_k, \hat{\mathbf C}_{k-1} \mathbf C_{k-1}^{-1}, \cdots, \hat{\mathbf C}_{k-m} \mathbf C_{k-m}^{-1}\right)    
\end{align}
where $\hat{\mathbf T}_k \mathbf T_k^{-1}\in SE_3(3)$ denotes the right invariant error given by
\begin{align}
\hat{\mathbf T}_k \mathbf T_k^{-1}&=
\left[
\begin{array}{c|ccc}
{^{G}_{I_{k}}{\mathbf {\widetilde R}}} & \bm\Gamma_1& \bm\Gamma_2 & \bm\Gamma_3\\
\hline
\mathbf 0_{3}&\multicolumn{3}{c}{\mathbf I_{3}}
\end{array}
\right]\nonumber\\
{^{G}_{I_{k}}{\mathbf {\widetilde R}}}&= {^{G}_{I_{k}}{\mathbf {\hat R}}}({^{G}_{I_{k}}{\mathbf R}})^{\top},\quad
\bm\Gamma_1={^G\hat{\mathbf v}}_{I_{k}} -{^{G}_{I_{k}}{\mathbf {\widetilde R}}} {^G{\mathbf v}_{I_{k}}} \nonumber\\
\bm\Gamma_2&={^G\hat{\mathbf p}}_{I_{k}} -{^{G}_{I_{k}}{\mathbf {\widetilde R}}} {^G{\mathbf p}_{I_{k}}}, \quad
\bm\Gamma_3={^G\hat{\mathbf p}}_{u} -{^{G}_{I_{k}}{\mathbf {\widetilde R}}} {^G{\mathbf p}_{f}},\nonumber
\end{align}
$\widetilde{\mathbf B}_k \in \mathbb R^6$ is the error of the IMU bias given as
\begin{align}\label{eq_err_b}
\widetilde{\mathbf B}_k&=\begin{bmatrix}
(\hat{\mathbf b}_{\omega_k}-\mathbf b_{\omega_k})^\top & (\hat{\mathbf b}_{a_k}-\mathbf b_{a_k})^\top
\end{bmatrix}
\end{align}
and $\hat{\mathbf C}_i {\mathbf C}_i^{-1}$, for $i\in\{k,\cdots, k-m\}$,
is defined in $SE(3)$ as 
\begin{align}
\hat{\mathbf C}_i {\mathbf C}_i^{-1}&=\begin{bmatrix}
{^{G}_{I_{i}}{\mathbf {\hat R}}}({^{G}_{I_{i}}{\mathbf R}})^{\top}&{^G\hat{\mathbf p}}_{I_{i}} -{^{G}_{I_{i}}{\mathbf {\hat R}}}({^{G}_{I_{i}}{\mathbf R}})^{\top} {^G{\mathbf p}_{I_{i}}} \\
\mathbf 0_{1\times3} & 1
\end{bmatrix}
\end{align}
By applying the log-linear property of the invariant error \cite{BA2017}, errors $\hat{\mathbf T}_k {\mathbf T}_k^{-1}$ and errors $\hat{\mathbf C}_i {\mathbf C}_i^{-1}$ can be approximated using a first-order approximation as follows:
\begin{align}\label{eq_logln}
\hat{\mathbf T}_k {\mathbf T}_k^{-1}&=\exp_\mathcal G(\bm \xi_{I_k})\approx \mathbf I_{6}+\bm {\xi}_{I_k}^\wedge\in\mathbb R^{6\times 6}\nonumber\\
\hat{\mathbf C}_i {\mathbf C}_i^{-1}&=\exp_\mathcal G(\bm \xi_{c_i})\approx \mathbf I_{6}+\bm {\xi}_{c_i}^\wedge\in\mathbb R^{4\times 4}
\end{align}
where $(\cdot)^\wedge: \mathbb R^{\text{dim} \mathfrak g}\to \mathfrak g$ be the linear map that transforms the error vector $\bm \xi_{I_k}$ and $\bm \xi_{c_k}$ defined in the Lie algebra to its corresponding matrix representation \cite{BA2017} as
\begin{align}\label{eq_err_lie1}
\bm \xi_{I_k}&\triangleq\begin{bmatrix}
(\bm \xi_{\theta_{k}})^\top & (\bm \xi_{v_{k}})^\top & (\bm \xi_{p_{k}})^\top & (\bm\xi_{u_k})^\top
\end{bmatrix}^\top \in \mathbb R^{12}\nonumber\\
\bm \xi_{\theta_{k}}&=\widetilde{\bm \theta}_{k}\nonumber=\log({^{G}_{I_{k}}{\mathbf {\widetilde R}}})\in\mathbb R^3\\
\bm \xi_{v_{k}}&={^G\hat{\mathbf v}}_{I_{k}}-(\mathbf I_3+\lfloor \widetilde{\bm \theta}_{k}\times\rfloor)^G \mathbf v_{I_{k}}\in\mathbb R^3\nonumber\\
\bm \xi_{p_{k}}&={^G\hat{\mathbf p}}_{I_{k}}-(\mathbf I_3+\lfloor \widetilde{\bm \theta}_{k}\times\rfloor)^G \mathbf p_{I_{k}}\in\mathbb R^3 \nonumber\\
\bm \xi_{u_k}&={^G\hat{\mathbf p}}_{u}-(\mathbf I_3+\lfloor \widetilde{\bm \theta}_{k}\times\rfloor)^G \mathbf p_{u}\in\mathbb R^3,
\end{align}
and 
\begin{align}\label{eq_err_lie2}
\bm \xi_{c_i}&\triangleq\begin{bmatrix}
(\bm \xi_{\theta_{i}})^\top & (\bm \xi_{p_{i}})^\top 
\end{bmatrix}^\top \in \mathbb R^{6}\nonumber\\
\bm \xi_{\theta_{i}}&=\widetilde{\bm \theta}_{k}\nonumber=\log({^{G}_{I_{i}}{\mathbf {\widetilde R}}})\in\mathbb R^3\\
\bm \xi_{p_{i}}&={^G\hat{\mathbf p}}_{I_{i}}-(\mathbf I_3+\lfloor \widetilde{\bm \theta}_{i}\times\rfloor)^G \mathbf p_{I_{i}}\in\mathbb R^3 
\end{align}
Give the error definitions in \eqref{eq_err_b}, \eqref{eq_err_lie1}, and \eqref{eq_err_lie2}, we define the overall error state vector at time $k$ as
\begin{align}\label{eq_err}
\widetilde{\mathbf x}_k&\triangleq\begin{bmatrix}
\widetilde{\mathbf x}_{I_k}^\top &\widetilde{\mathbf x}_{C}^\top 
\end{bmatrix}^\top \in \mathbb R^{18+6m}\nonumber\\
\widetilde{\mathbf x}_{I_k}&=
\begin{bmatrix}
\bm \xi_{I_k}^\top & \widetilde{\mathbf B}_k^\top 
\end{bmatrix}^\top, \quad \widetilde{\mathbf x}_{C}=\begin{bmatrix}
\bm\xi_{c_{k-1}}^\top & \cdots & \bm\xi_{c_{k-m}}^\top    
\end{bmatrix}^\top
\end{align}
\subsection{System Model}
\noindent\textbf{System motion model:}
In the proposed CVIRO, we adopt the general 3-D rigid-body kinematics as the system motion model, given by:
\begin{align}\label{eq_imu_kn}
{^{G}_{I_{k}}\dot{\mathbf R}}&={^{G}_{I_{k}}{\mathbf R}}\lfloor {^{I_{k}}\bm\omega} \times\rfloor,\quad
{^G{\dot{\mathbf v}}_{I_{k}}}={^{G}_{I_{k}}{\mathbf R}} ({^{I_{k}}\mathbf a})+{^G\mathbf g}\nonumber\\
{^G{\dot{\mathbf p}}_{I_{k}}}&={^G{\mathbf v}_{I_{k}}},\quad
{^G{\dot{\mathbf p}}_{u}}=\mathbf 0,\quad
\dot{\mathbf b}_{\omega_{k}}=\mathbf w_{\omega}, \quad \dot{\mathbf b}_{a_{k}}=\mathbf w_{a}
\end{align}
where $^G \mathbf g=\begin{bmatrix}0&0&-9.8
\end{bmatrix}^\top$ denotes the gravity vector.
In particular, the IMU sensor can measure the robot's acceleration ${^{I_{k}}\mathbf a}$ and angular velocity ${^{I_{k}}\bm\omega}$ with respect to the IMU's own frame $I_k$ at each timestep $k$ as
\begin{align}
{{^{I_{k}}\mathbf a}_m}&={^{I_{k}}\mathbf a}-{^{G}_{I_{k}}{\mathbf R}}^\top {^G \mathbf g}+ {\mathbf n}_{a_{k}}+\mathbf b_{a_{k}}\nonumber\\
{^{I_{k}}\bm\omega}_m&={^{I_{k}}\bm\omega}+\mathbf n_{\omega_{k}}+\mathbf b_{\omega_{k}}
\end{align}
where ${{^{I_{k}}\mathbf a}_m}$ and ${^{I_{k}}\bm\omega}_m$ are the noisy measurements of the robot's acceleration and angular velocity, respectively.
$\mathbf n_{a_k}$ and $\mathbf n_{\omega_k}$ are the IMU's measurement noises assumed to be zero-mean Gaussian. $\mathbf b_{\omega_{k}}$ and $\mathbf b_{a_{k}}$ are the IMU's biases which are modeled as random walk process, where their time derivatives follow a Gaussian distribution.

\noindent\textbf{Camera measurement model:}
When a static landmark of the environment, denoted as ${^{G}\mathbf p_f}$, is tracked by the camera at timestep $k$, the corresponding feature measurement can be obtained through the following model
\begin{align}
\mathbf z_{c_k}&=\prod({^{C_k}\mathbf p_f})+\mathbf n_{c_k}\nonumber\\
{^{C_k}\mathbf p_f}&={^{C}_{I} \mathbf R}{^{I_k}_G \mathbf R}({^{G}\mathbf p_f}-{^{G}\mathbf p_{I_k}})+{^{C}\mathbf p_{I}}
\end{align}
where $\mathbf n_{c_k}\sim \mathcal N(0, \mathbf Q_c)$ is the measurement noise which is assumed to be white Gaussian with covariance $\mathbf Q_c$. The ${^{C_k}\mathbf p_f}$ is the landmark position in the camera frame, and the projection function $\prod(\cdot)$ is defined as $\prod(\begin{bmatrix}x&y&z\end{bmatrix}^\top)=\begin{bmatrix}x/z&y/z
\end{bmatrix}^\top$

\noindent\textbf{Range measurement model:}
The UWB tag on the robot can provide the ranging measurement 
between the robot and the UWB anchor at time $t_k$, denoted as $d_{k}$. Given the robot's pose $({^{I_k}_G{\mathbf R}}, ^G\mathbf p_{I_k})$ and the anchor position ${^G\mathbf p_{u}}$, the ranging measurement is described by the following model:
\begin{align}\label{eq_uwb}
\mathbf z_{u_k}&=\|^G\mathbf p_{I_k}+{^{I_k}_G{\mathbf R}}^\top{^{I}\mathbf p_T}-{^G\mathbf p_{u}}\|+\mathbf n_{u_k} + \mathbf b_{u_k}
\end{align}
where $\mathbf n_{u_k}\sim \mathcal N(0, \mathbf Q_u)$ is the measurement noise; ${^{I}\mathbf p_T}$, $\mathbf b_{u_k}$ represent UWB tag's position in the IMU frame and the measurement bias, respectively. Note that the terms, ${^{I}\mathbf p_T}$ and $\mathbf b_{u_k}$ can be easily calibrated offline \cite{ntuviral}.

\section{Algorithm Design}
\subsection{IMU Propagation}
Consider a posterior estimate at time $t_k$, denoted as $(\hat{\mathbf X}_k, \hat{\mathbf P}_{k})$. We propagate the system state $\hat{\mathbf X}_k$ forward to compute a prior estimate at time $t_{k+1}$, denoted as $(\hat{\mathbf X}_{k+1|k}, \hat{\mathbf P}_{k+1|k})$, by integrating the IMU measurements between time $t_k$ and time $t_{k+1}$. 
To propagate the state covariance, we first linearize the continuous-time kinematics \eqref{eq_imu_kn} at $\hat {\mathbf X}_k$ to compute a linearized error-state model as
\begin{align}\label{eq_lin}
\frac{d}{dt}{\widetilde{\mathbf x}}_{k}&=
\begin{bmatrix}
\mathbf F_k & \mathbf 0_{18\times 6m}\\
\mathbf 0_{6m\times18} & \mathbf 0_{6m}
\end{bmatrix}
{\widetilde{\mathbf x}}_{k}+
\begin{bmatrix}
\mathbf G_k \\ \mathbf 0_{6m\times18}
\end{bmatrix}\mathbf n_k,
\end{align}
where $\mathbf n_k=\begin{bmatrix}
\mathbf n_{\omega_k}^\top & \mathbf n_{a_k}^\top & \mathbf 0_{1\times 6}^\top & \mathbf w_{\omega}^\top & \mathbf w_{a}^\top
\end{bmatrix}^\top$; The jacobian $\mathbf F_k$ can be computed as
\begin{align}
\mathbf F_k&=
\begin{bmatrix}
\mathbf F_A & \mathbf F_{B_k}\\
\mathbf 0_{6\times 12} & \mathbf 0_{6}
\end{bmatrix}, \quad
\mathbf F_{A}=\begin{bmatrix}
\mathbf 0_3 & \mathbf 0_3 & \mathbf 0_3 & \mathbf 0_3\\
\lfloor^G\mathbf{g}\times\rfloor & \mathbf 0_3 & \mathbf 0_3 & \mathbf 0_3 \\
\mathbf 0_3 & \mathbf{I}_3 & \mathbf 0_3 & \mathbf 0_3 \\
\mathbf 0_3 & \mathbf 0_3 & \mathbf 0_3 & \mathbf 0_3 \\
\end{bmatrix}\nonumber\\
\mathbf F_{B_{k}}&=\begin{bmatrix}
-{^G_{I_{k}} \hat{\mathbf R}} & \mathbf 0_3\\
-\lfloor{^G\hat{\mathbf v}_{I_{k}} \times}\rfloor{^G_{I_{k}} \hat{\mathbf R}} & -{^G_{I_{k}} \hat{\mathbf R}} \\
-\lfloor{^G\hat{\mathbf p}_{I_{k}} \times}\rfloor{^G_{I_{k}} \hat{\mathbf R}} & \mathbf 0_3 \\
-\lfloor{^G\hat{\mathbf p}_{u} \times}\rfloor{^G_{I_{k}} \hat{\mathbf R}} & \mathbf 0_3
\end{bmatrix}
\end{align}
and the jacobian $\mathbf G_k$ can be computed as
\begin{align}
\mathbf G_k&=
\begin{bmatrix}
\mathbf{Ad}_{\hat {\mathbf X}_{k}} & \mathbf 0_{12\times 6}\\
\mathbf 0_{6\times 12} & \mathbf 0_6
\end{bmatrix}\nonumber\\
\mathbf{Ad}_{\hat {\mathbf X}_{k}}&=\begin{bmatrix}
{^{G}_{I_{k}}\hat{\mathbf R}}& \mathbf 0_3 & \mathbf 0_3&\mathbf 0_3 \\
\lfloor{^G\hat{\mathbf v}_{I_{k}}\times \rfloor{^{G}_{I_{k}}}\hat{\mathbf R}}& {^{G}_{I_{k}}\hat{\mathbf R}} &\mathbf 0_3 & \mathbf 0_3 \\
\lfloor{^G\hat{\mathbf p}_{I_{k}}\times \rfloor{^{G}_{I_{k}}}\hat{\mathbf R}}& \mathbf 0_3& {^{G}_{I_{k}}\hat{\mathbf R}} & \mathbf 0_3 \\
\lfloor{^G\hat{\mathbf p}_{u}\times \rfloor{^{G}_{I_{k}}}\hat{\mathbf R}}& \mathbf 0_3& \mathbf 0_3 & {^{G}_{I_{k}}\hat{\mathbf R}}
\end{bmatrix}
\end{align}
Given the linearized system in \eqref{eq_lin}, we can propagate the covariance as 
\begin{align}
\hat{\mathbf P}_{k+1|k} &= \bar{\bm\Phi}_{k+1|k}\hat{\mathbf P}_{k|k}\bar{\bm\Phi}_{k+1|k}^\top + \mathbf Q
\end{align}
where $\mathbf Q$ denotes the noise covariance matrix, and the discrete-time state transition matrix from time $t_{k}$ to $t_{k+1}$ can be computed as
\begin{align}
\bar{\bm\Phi}_{k+1|k}&=\text{diag}({\bm\Phi}_{k+1|k}, \mathbf I_{6m}),\quad {\bm\Phi}_{k+1|k}=\begin{bmatrix}
\bm\Phi_{A}& \bm\Phi_{B_{k}} \\
\mathbf 0_{6\times 12} & \mathbf I_{6}
\end{bmatrix}.
\end{align}
For the convenience of readers, we provide a more detailed derivation of the IMU propagation in the {\textbf{Section 2}} of the supplementary material \cite{zzSPLY2025}.

\subsection{Range Update}
When the robot's UWB tag receives the range measurement from the anchor, this measurement will be employed to update the system state, which includes both the robot's state and the UWB anchor's state. To perform the range update, we first linearize the range measurement model \eqref{eq_uwb} by using the first-order Taylor expansion. Notably, unlike the standard Taylor expansion represented with the standard vector error, we have to linearize \eqref{eq_uwb} with the right-invariant error $\widetilde{\mathbf x}_{I_k}$ in \eqref{eq_err}. Recall that $\exp_{\mathcal G}(\bm \xi_{I_k})\approx \mathbf I_6 +\bm \xi_{I_k}^\wedge$ 
in \eqref{eq_logln}, we have the following linearized model at the linearization point $\hat{\mathbf X}_k$ as
\begin{align}\label{eq_ln_uwb}
\widetilde{\mathbf z}_{u_k}&= \mathbf H_{u_k} \widetilde{\mathbf x}_{I_k} + \mathbf n_{u_k}
\end{align}
where the measurement jacobian $\mathbf H_{u_k}$ is given by
\begin{align}
\mathbf H_{u_k}&= \mathbf H_{pu}\begin{bmatrix}
\Lambda_k & \mathbf 0_3 & \mathbf I_3 & -\mathbf I_3 & \mathbf 0_{3\times 6}
\end{bmatrix}\nonumber\\
\mathbf H_{pu}&=\frac{ \left({^G \hat{\mathbf p}_{I_k}}-{^G \hat{\mathbf p}_u}+{^G_{I_k} \hat{ \mathbf R}}\, {^I\mathbf p_T} \right)^\top}{\|  {^G \hat{\mathbf p}_{I_k}}-{^G \hat{\mathbf p}_u}+{^G_{I_k} \hat{ \mathbf R}}\, {^I\mathbf p_T}\|} \nonumber\\
\Lambda_k&= \lfloor{ {^G \hat{\mathbf p}_u}-{^G \hat{\mathbf p}_{I_k}}-{^G_{I_k} \hat{ \mathbf R}}\, {^I\mathbf p_T}\times\rfloor}.
\end{align}
More detailed derivation are given in the \textbf{Section 3} of the supplementary material \cite{zzSPLY2025}.

\subsection{Visual Update}
Similar to the derivation of range update, to linearize the camera measurement model at $\hat {\mathbf X}_k$, we perform the first-order Taylor expansion to the above function as
\begin{align}
\widetilde{\mathbf z}_{c_k}&= \mathbf H_{x_k} \widetilde{\mathbf x}_{I_k} + \mathbf H_{f_k} {^G\widetilde{\mathbf p}}_{f} + \mathbf n_{c_k}
\end{align}
where the measurement jacobians $\mathbf H_{x_k}$ and $\mathbf H_{f_k}$ can be computed as
\begin{align}
\mathbf H_{x_k}&= \mathbf H_{pc}{^C_{I}\hat{\mathbf R}}{^{I_k}_G\hat{\mathbf R}}\begin{bmatrix} \mathbf 0_3 & \mathbf 0_3 & -\mathbf I_3 & \mathbf 0_{3\times 9}
\end{bmatrix}\nonumber\\
\mathbf H_{f_k}&=\mathbf H_{pc}{^C_{I}\hat{\mathbf R}}{^{I_k}_G\hat{\mathbf R}}, \quad
\mathbf H_{pc}=\begin{bmatrix}
1/\hat{z}& 0 & -\hat x/\hat{z}^2\\
0 & 1/\hat{z} & -\hat y/\hat{z}^2
\end{bmatrix}.
\end{align}
Since the feature $^G \mathbf p_f$ is not included in the state, we project the residuals onto the left null space of $\mathbf H_{f_k}$ as the following to perform update
\begin{align}
\widetilde{\mathbf z}_{c_k}^\prime&=\mathbf L_f^\top\widetilde{\mathbf z}_{c_k}=\mathbf L_f^\top \mathbf H_{x_k} \widetilde{\mathbf x}_{I_k} + \mathbf L_f^\top\mathbf n_{c_k}
\end{align}
where the columns of matrix $\mathbf L_f$ forms a basis of
the left nullspace $\mathbf H_{f_k}$.

\subsection{UWB state augmentation and initialization}
Since the UWB anchor positions are unknown and need to be estimated in the proposed CVIRO framework, one crucial task is to initialize the UWB state as well as the estimated covariance. Specifically, when the system starts to run, the state estimate $\hat{\mathbf T}_k$ only contains the IMU state without the anchor position ${^G \hat{ \mathbf p}_u}$. We assume the IMU state, as obtained through the VIO system, along with the ranging measurement, are readily accessible and stored in a time window of length $n$. To initialize the anchor state, the formulate the following optimization problem
\begin{align}
\min_{^G \mathbf p_u}\sum_{k=1}^n\left(\mathbf z_{u_k} -\mathbf b_{u_k}- \| {^G\mathbf p_{T_k}}-{^G\mathbf p}_{u}\|\right)^2,
\end{align}
where ${^G\mathbf p_{T_k}}={^G\mathbf p}_{I_k}+ {^{I_k}_G \mathbf R^\top} {^I\mathbf p}_{T}$ denotes the UWB tag's position with respect to the IMU frame $I_k$.

To initialize the covariance of the UWB state as well as its correlations with existing states, we adopt a method closely resembling the "state variable initialization" described in \cite{OVS}. We first define a state $\mathbf X_{pos}$ that contains the robot poses from timestep $k=0$ to $k=n$ as 
\begin{align}\mathbf X_{pos}=
\begin{bmatrix}
{^G_{I_0}\mathbf R}^\top &{^G\mathbf p}_{I_0}^\top & \cdots &{^G_{I_n}\mathbf R}^\top &{^G\mathbf p}_{I_n}^\top
\end{bmatrix}^\top,
\end{align}
and then stack all the available UWB measurements $\mathbf z_u=\begin{bmatrix}
\mathbf z_{u_0}, \cdots, \mathbf z_{u_n}
\end{bmatrix}^\top$ to construct a stacked measurement model as
\begin{align}
\mathbf z_u&=\mathbf h(\mathbf X_{pos}, {^G \mathbf p}_{u})
\end{align}
where each $\mathbf z_{u_k}$ satisfies the model \eqref{eq_uwb}. By linearizing the above measurement model, we have
\begin{align}
\widetilde{\mathbf z}_{u}&=\begin{bmatrix}
\mathbf H_A & \mathbf H_B 
\end{bmatrix}\begin{bmatrix}
\widetilde{\mathbf x}_{pos} \\ \bm \xi_{u_k}
\end{bmatrix}+\bar{\mathbf n}_u
\end{align}
where $\bar{\mathbf n}_u$ represents the stacked noise vector and $\widetilde{\mathbf x}_{pos}$ denotes the error state of ${\mathbf x}_{pos}$ given by
\begin{align}
\widetilde{\mathbf x}_{pos}=\begin{bmatrix}
\bm \xi_{\theta_0}^\top & \bm \xi_{p_0}^\top & \cdots & \bm \xi_{\theta_n}^\top & \bm \xi_{p_n}^\top
\end{bmatrix}^\top.
\end{align}
$\mathbf H_A$ and $\mathbf H_B$ are the corresponding measurement jacobians, which can be computed in a manner similar to the Jacobian $\mathbf H_u$ in \eqref{eq_ln_uwb}.
We perform QR decomposition to separate the above linearized system as two subsystems as
\begin{align}
\begin{bmatrix}
\widetilde{\mathbf z}_{u_1} \\ \widetilde{\mathbf z}_{u_2}
\end{bmatrix}&=\begin{bmatrix}
\mathbf H_{A_1} & \mathbf H_{A_2} \\ \mathbf H_{B_1} & \mathbf 0
\end{bmatrix} \begin{bmatrix}
\widetilde{\mathbf x}_{pos} \\ \bm \xi_{u_k}
\end{bmatrix} + \begin{bmatrix}
\bar{\mathbf n}_{u_1}\\ \bar{\mathbf n}_{u_2}
\end{bmatrix}
\end{align}
Then, the covariance of the UWB state and its correlations to the existing state $\mathbf X_{pos}$ can be computed and used to augmented to the current covariance following \cite{OVS}.


\section{Consistency Analysis}
\subsection{Observability Analysis}\label{ob_ana}
Observability is fundamental for analyzing system consistency, as it determines whether the system can fully recover its initial state using all available measurements. According to \cite{LIVIO2015}, the observability of the MSCKF system can be analyzed through the EKF-SLAM framework, as the two are theoretically equivalent. Thus, without loss of generality, we study the observability of the EKF-SLAM system with a single feature point $^G\mathbf{p}_{f}$ and one UWB anchor $^G\mathbf{p}_{u}$. The state of the EKF-SLAM system has a form very similar to $\mathbf{X}_{I_k}$ in \eqref{eq_state}, with the only difference being that the original state $\mathbf{T}_k$ is now augmented with a feature state, given by
\begin{align}
\mathbf T_{k}&=\left[
\begin{array}{c|cccc}
{^{G}_{I_{k}}\mathbf{R}} & ^G\mathbf{v}_{I_{k}} & ^G\mathbf{p}_{I_{k}} & ^G\mathbf{p}_{u} & ^G\mathbf{p}_{f}\\
\hline
\mathbf 0_{4\times 3}&\multicolumn{4}{c}{\mathbf I_{4}}
\end{array}
\right]\in\ SE_{4}(3), 
\end{align}
Note that, based on the definition of the Special Euclidean Group, this analysis can be easily extended to the multi-feature and multi-anchor case by augmenting the system state accordingly while preserving the group structure. Following the proof of \cite{FEJ}, the local observability matrix for the time-varying error state of the whole system is defined as
\begin{align}\label{eq_ob}
    \mathcal O&=\begin{bmatrix}
    \mathcal O_0\\\mathcal O_1\\ \vdots \\ \mathcal O_k
    \end{bmatrix}=\begin{bmatrix}
    \mathbf H_0 \\
    \mathbf H_1 \widetilde{\bm\Phi}_{1|0}\\
    \vdots\\
    \mathbf H_k \widetilde{\bm\Phi}_{k|0}
    \end{bmatrix},    
\end{align}
where $\mathbf H_k=\begin{bmatrix} \mathbf H_{c_k}^\top& \mathbf H_{r_k}^\top \end{bmatrix}^\top$ is the joint measurement jacobian; $\mathbf H_{c_k}$ represents the jacobian of the visual measurement, and $\mathbf H_{r_k}$ represents the jacobian of the range measurement
\begin{align}
\mathbf H_{c_k}&=\mathbf H_{f_k}\begin{bmatrix} \mathbf 0_3 & \mathbf 0_3 & -\mathbf I_3 & \mathbf 0_3 & \mathbf I_3 & \mathbf 0_{3\times 6}
\end{bmatrix}\nonumber\\
\mathbf H_{r_k}&=\mathbf H_{pu}\begin{bmatrix} \Lambda_k & \mathbf 0_3 & \mathbf I_3 & -\mathbf I_3 & \mathbf 0_3 & \mathbf 0_{3\times 6}
\end{bmatrix},
\end{align}
$\widetilde{\bm\Phi}_{k,0}$ is the state transition matrix of the EKF-SLAM system.
Then we can compute each block row $\mathcal O_k$ of the observability matrix as
\begin{align}
\mathcal O_k&\triangleq\begin{bmatrix}
\mathcal O_{c_k}\\\mathcal O_{r_k}
\end{bmatrix}=\begin{bmatrix}
\mathbf H_{c_k} \widetilde{\bm\Phi}_{k,0}\\ \mathbf H_{r_k} \widetilde{\bm\Phi}_{k,0}
\end{bmatrix}\nonumber\\
\mathcal O_{c_k}&=\mathbf H_{f_k}\begin{bmatrix}
\mathbf M_{c,1} & -\mathbf I_3 \delta t & -\mathbf I_3 & \mathbf 0_3 & \mathbf I_3 & \mathbf M_{c,2} & \mathbf M_{c,3}
\end{bmatrix}\nonumber\\
\mathcal O_{r_k}&=\mathbf H_{pu}\begin{bmatrix}
\mathbf M_{r,1} & \mathbf I_3 \delta t & \mathbf I_3 & -\mathbf I_3 & \mathbf 0_3 & \mathbf M_{r,2} & \mathbf M_{r,3}
\end{bmatrix},
\end{align}
where 
\begin{align}
\mathbf M_{c,1}&=-\frac{1}{2}\lfloor^G \mathbf g \times\rfloor \delta t^2,\quad \mathbf M_{c,2}={\bm{ \widetilde \Phi}}_{56}-{\bm{ \widetilde \Phi}}_{36}\nonumber\\
\mathbf M_{c,3}&=-{\bm{ \widetilde \Phi}}_{37},\quad \mathbf M_{r,1}=\Lambda_k+\frac{1}{2}\lfloor^G \mathbf g \times\rfloor \delta t^2,\nonumber\\
\mathbf M_{r,2}&=\Lambda_k{\bm{ \widetilde \Phi}}_{16}+{\bm{ \widetilde \Phi}}_{36}-{\bm{ \widetilde \Phi}}_{46},\quad \mathbf M_{r,3}={\bm{ \widetilde \Phi}}_{37}
\end{align}
and $\widetilde{\bm \Phi}_{ij}$ are sub-block matrices of the state transition matrix $\widetilde{\bm \Phi}_{k,0}$. The exactly forms of the sub-block matrix $\widetilde{\bm \Phi}_{ij}$ are given in the \textbf{Section 4} of the supplementary material \cite{zzSPLY2025}.

\begin{theorem}\label{OBA}\longthmtitle{Observability properties of CVIRO}
\textit{The right nullspace of the observability matrix $\mathcal O_k$, denoted as $\mathbf N$, spans the following unobservable subspace
\begin{align}
\mathbf N=\begin{bmatrix}
{^G \mathbf g}^\top & 0 & 0 & 0 & 0 & 0 & 0\\
\mathbf 0_3 & \mathbf 0_3 & \mathbf I_3 & \mathbf I_3 & \mathbf I_3 & \mathbf 0_3 &\mathbf 0_3
\end{bmatrix}^\top
\end{align}
which denotes the global orientation and yaw translation.}
\end{theorem}
\begin{proof}
See \textbf{Section 4} of the supplementary material for detailed derivations \cite{zzSPLY2025}.

\end{proof}
It can be concluded from Theorem \ref{OBA} that, unlike the standard EKF-based VIO system, where the unobservable space depends on the state estimate \cite{FEJ}, the nullspace $\mathbf N$ remains a constant matrix, independent of the system's state estimate. This nice property ensures that the CVIRO system naturally maintains the correct observability properties and does not introduce any spurious information along the unobservable subspace. Therefore, unlike standard EKF-based methods, which require artificially modifying the linearization point during estimation to enforce correct observability properties \cite{FEJ}, the CVIRO system is inherently observability consistent.

\section{Simulations and Experiments}
In this section, we conduct both Monte Carlo simulations and real-world experiments to evaluate the effectiveness and performance of the proposed CVIRO algorithm. The evaluation metrics used to assess estimation performance include Root Mean Square Error (RMSE), Absolute Trajectory Error (ATE) for accuracy and Normalized Estimation Error Squared (NEES) for consistency.

\begin{figure}[htbp]
    \centering
    \begin{subfigure}[b]{0.15\textwidth}
        \centering
        \includegraphics[width=\textwidth]{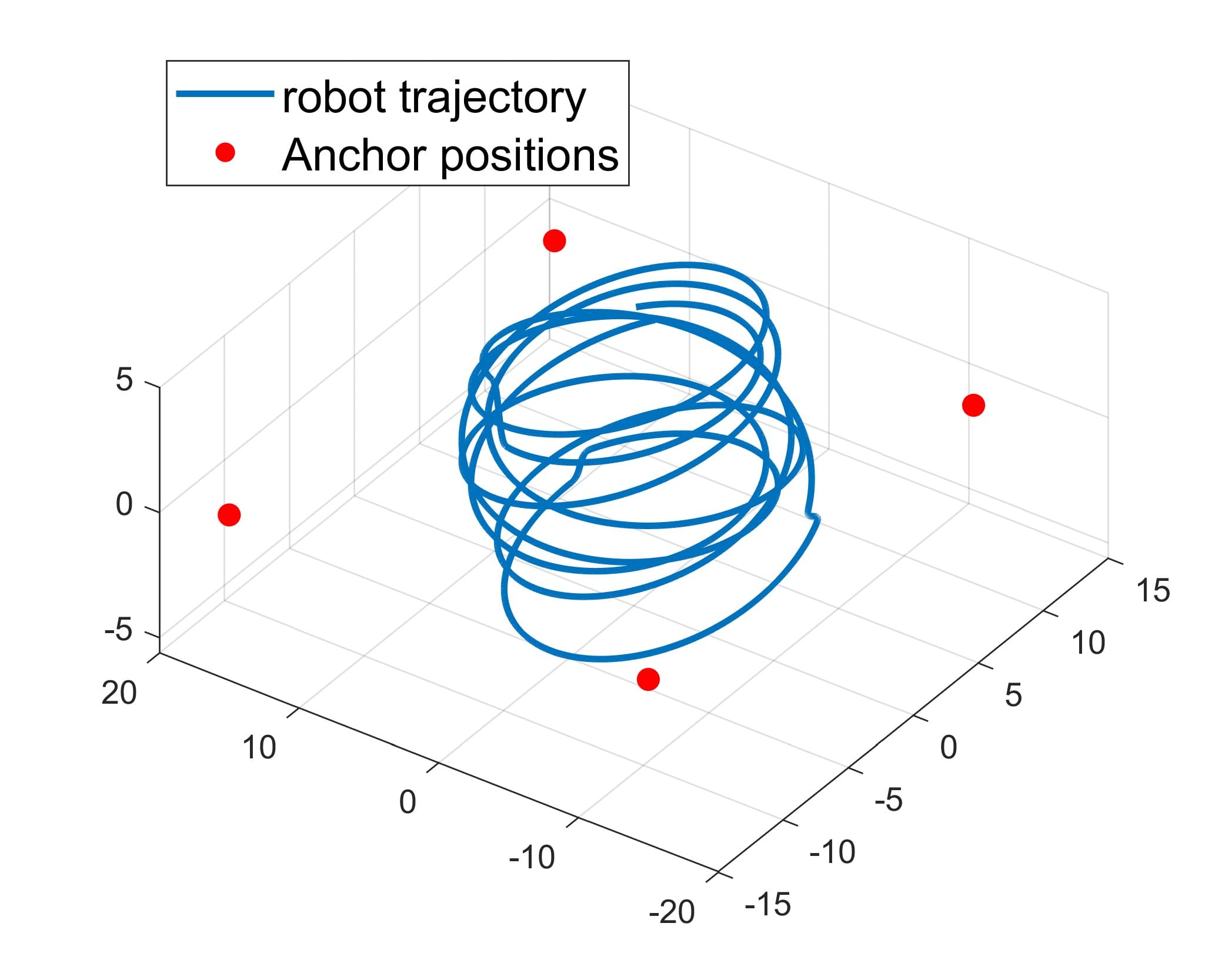}
        \caption{}
        \label{fig:traj1}
    \end{subfigure}
    \begin{subfigure}[b]{0.15\textwidth}
        \centering
        \includegraphics[width=\textwidth]{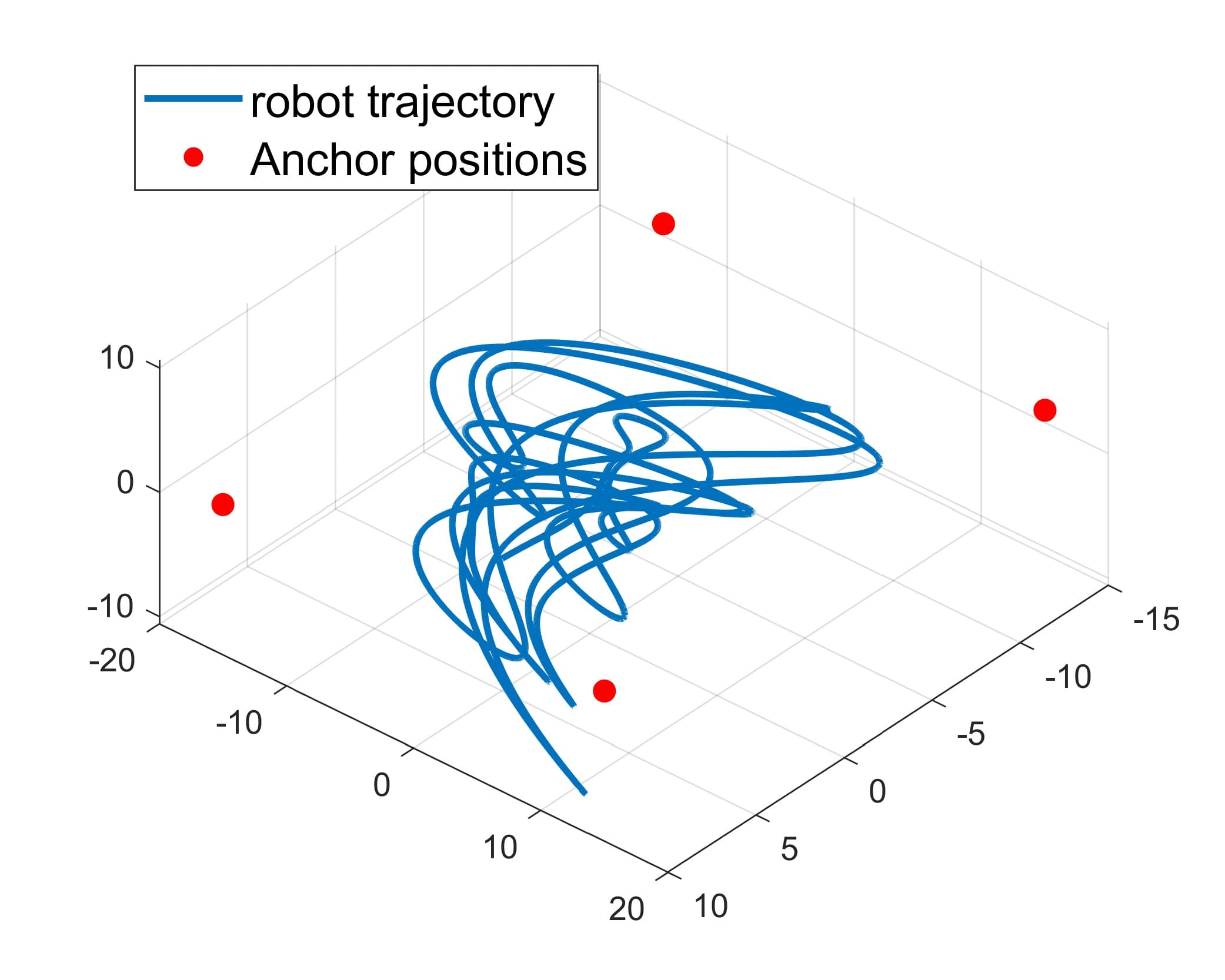}
        \caption{}
        \label{fig:traj2}
    \end{subfigure}
    \begin{subfigure}[b]{0.15\textwidth}
        \centering
        \includegraphics[width=\textwidth]{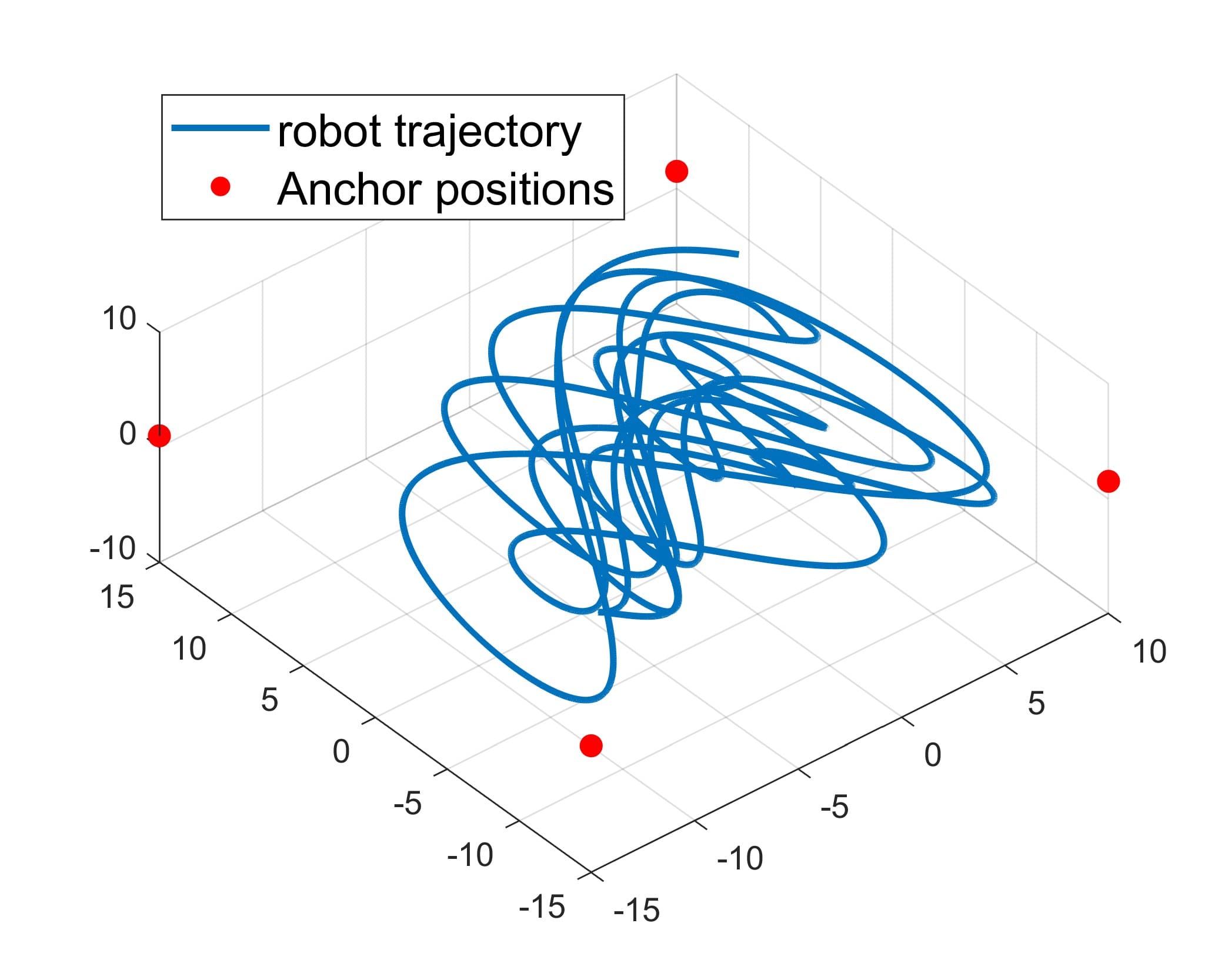}
        \caption{}
        \label{fig:traj3}
    \end{subfigure}
    \caption{Simulated trajectories with four UWB anchors placed in the environment. The lengths of the simulated trajectories are as follows: (a) 405 m, (b) 510 m, and (c) 542 m.}
    \label{fig_traj}
\end{figure}

\vspace{-1.5ex}
\begin{table}[h]
   \centering
   \caption{Simulation parameters used in Monte Carlo experiments.}
   \label{Sim_para}
   \begin{tabular}{cccc}
      \toprule[1.2pt]
      \textbf{Parameter} & \textbf{Value} & \textbf{Parameter} & \textbf{Value}  \\
      \midrule[1.2pt]
      Cam Noise (pixel) & 1 & UWB Noise (m) & 0.10 \\
      Gyro Noise ($\text{rad}/(\text{s}\sqrt{\text{Hz}})$) & 2.0e-3 & IMU freq. (\text{Hz}) & 100 \\
      Accel Noise ($\text{m}/(\text{s}^2\sqrt{\text{Hz}})$) & 3.0e-3 & Cam freq. (\text{Hz}) & 10 \\
      Gyro Bias ($\text{rad}/(\text{s}\sqrt{\text{Hz}})$) & 3.0e-4 & UWB freq. (\text{Hz}) & 10 \\
      Accel Bias ($\text{m}/(\text{s}^2\sqrt{\text{Hz}})$) & 3.0e-4 & Max Clones & 11 \\
      \bottomrule[1.2pt]
   \end{tabular}
\end{table}

\subsection{Numerical simulations}
In the simulations, we configure the robot to follow a pre-designed trajectory and place four UWB anchors around the environment as shown in Figure \ref{fig_traj}. Each UWB anchor measures the distance to the tag attached to the robot. Specifically, we compare our proposed method against three baseline estimators: (1) standard MSCKF-based VIO, which does not incorporate UWB measurements; (2) MSCKF-LG \cite{Heo2018}, a state-of-the-art Lie group-based VIO algorithm; and (3) FEJ-VIRO \cite{Jia2022}, a tightly-coupled VIO approach incorporating UWB measurements with the First-Estimate Jacobian technique. We select these three baseline estimators to comprehensively evaluate our proposed CVIRO method. The standard MSCKF-based VIO demonstrates the necessity of incorporating UWB measurements. MSCKF-LG serves as a state-of-the-art Lie group-based estimator that achieves observability consistency but without UWB integration. FEJ-VIRO represents a tightly-coupled method that integrates UWB measurements and addresses observability consistency through the First-Estimate Jacobian technique, enabling a fair comparison regarding both UWB fusion and consistency. We perform 50 runs of Monte Carlo simulations, with the simulation parameters detailed in Table \ref{Sim_para}. All algorithms are tested under identical parameters to ensure a fair comparison. Moreover, since our algorithm is based solely on MSCKF features and does not incorporate SLAM features, we disable SLAM feature tracking in the other algorithms (VIRO-FEJ) as well to maintain a fair evaluation.


The proposed CVIRO algorithm is evaluated in terms of both accuracy and consistency to comprehensively assess its performance across all three simulated trajectories. As illustrated in Figures \ref{fig:subfig2}, \ref{fig:subfig4}, and \ref{fig:subfig6}, our method, MSCKF-LG, and FEJ-VIRO are consistent, whereas the standard MSCKF-based VIO demonstrates clear inconsistency. The standard MSCKF is inconsistent because it fails to preserve the correct observability properties of the original system, whereas the other methods explicitly maintain observability consistency. Notably, our method and LG-MSCKF employ a Lie group-based state representation and leverage the invariant error property, inherently preserving the correct unobservable subspace of the original system, whereas FEJ-VIRO utilizes the First-Estimate Jacobian (FEJ) method to ensure observability consistency. Moreover, by directly incorporating the UWB state into the system state and performing estimation in a tightly-coupled manner, our method and FEJ-VIRO explicitly account for UWB uncertainties, thereby ensuring estimator consistency. In addition, our method and FEJ-VIRO demonstrate significant improvements in terms of accuracy by incorporating the UWB measurements to enhance the VIO localization performance as shown in Figure \ref{fig:subfig1}, \ref{fig:subfig3}, and \ref{fig:subfig5}.
We observe that our CVIRO algorithm and the FEJ-VIRO algorithm achieve comparable localization accuracy when the robot follows a smooth and less aggressive trajectory, as shown in Figure \ref{fig:traj1}. But, our method demonstrates superior performance during aggressive motion, as seen in Figures \ref{fig:traj2} and \ref{fig:traj3}. One possible reason for this discrepancy is that, as discussed earlier, although the FEJ method ensures observability consistency, it is prone to large initialization errors due to its reliance on the first-estimate Jacobian for linearization. If the initial linearization point is inaccurate, this can introduce additional estimation errors. When the robot undergoes aggressive motion, the quality of UWB initialization further degrades, potentially compromising the performance of the FEJ-VIRO algorithm.

\begin{figure}[htbp]
    \centering
    \begin{subfigure}[b]{0.23\textwidth}
        \centering
        \includegraphics[width=\textwidth]{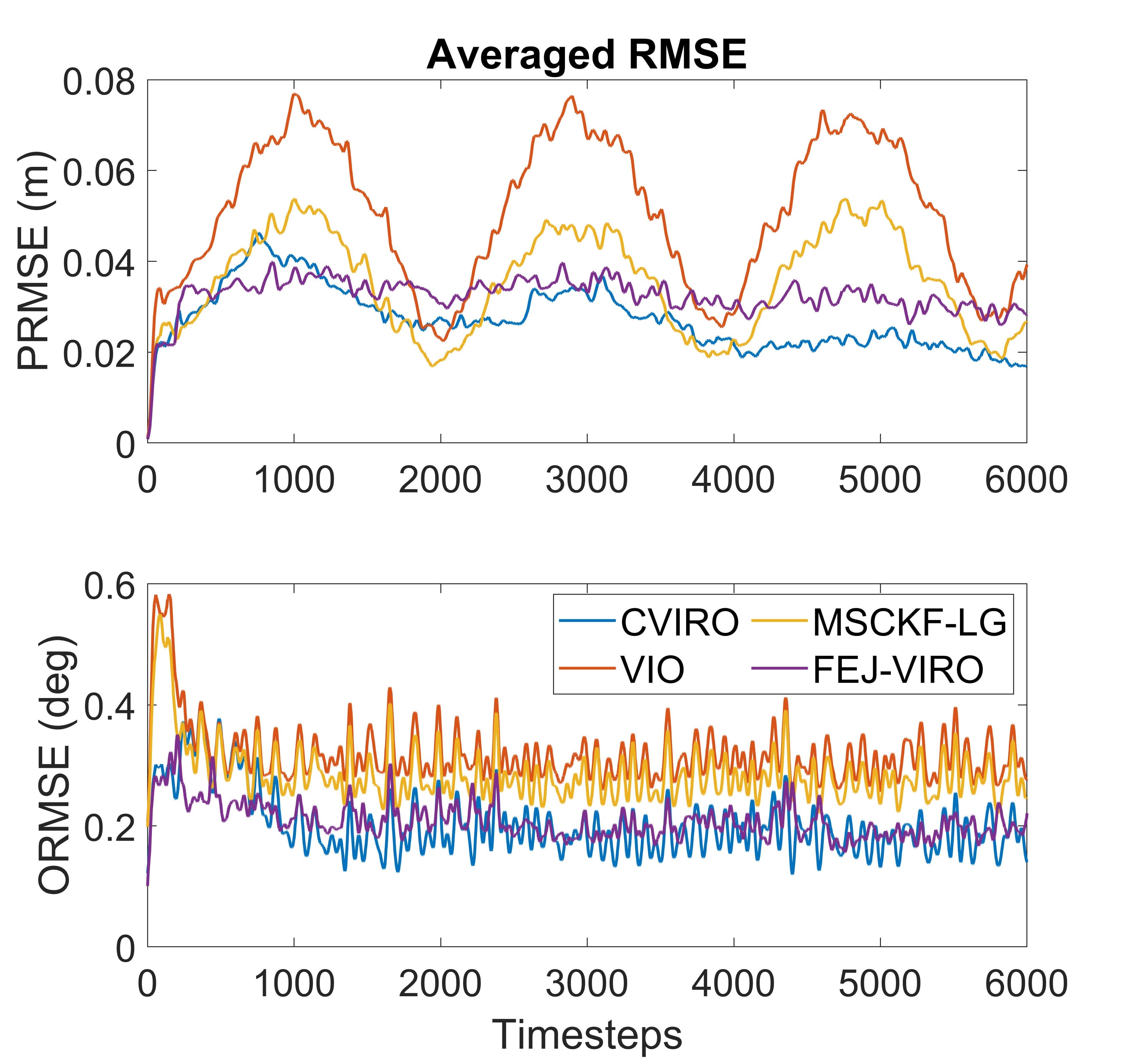}
        \caption{RMSE of traj. a}
        \label{fig:subfig1}
    \end{subfigure}
    \begin{subfigure}[b]{0.23\textwidth}
        \centering
        \includegraphics[width=\textwidth]{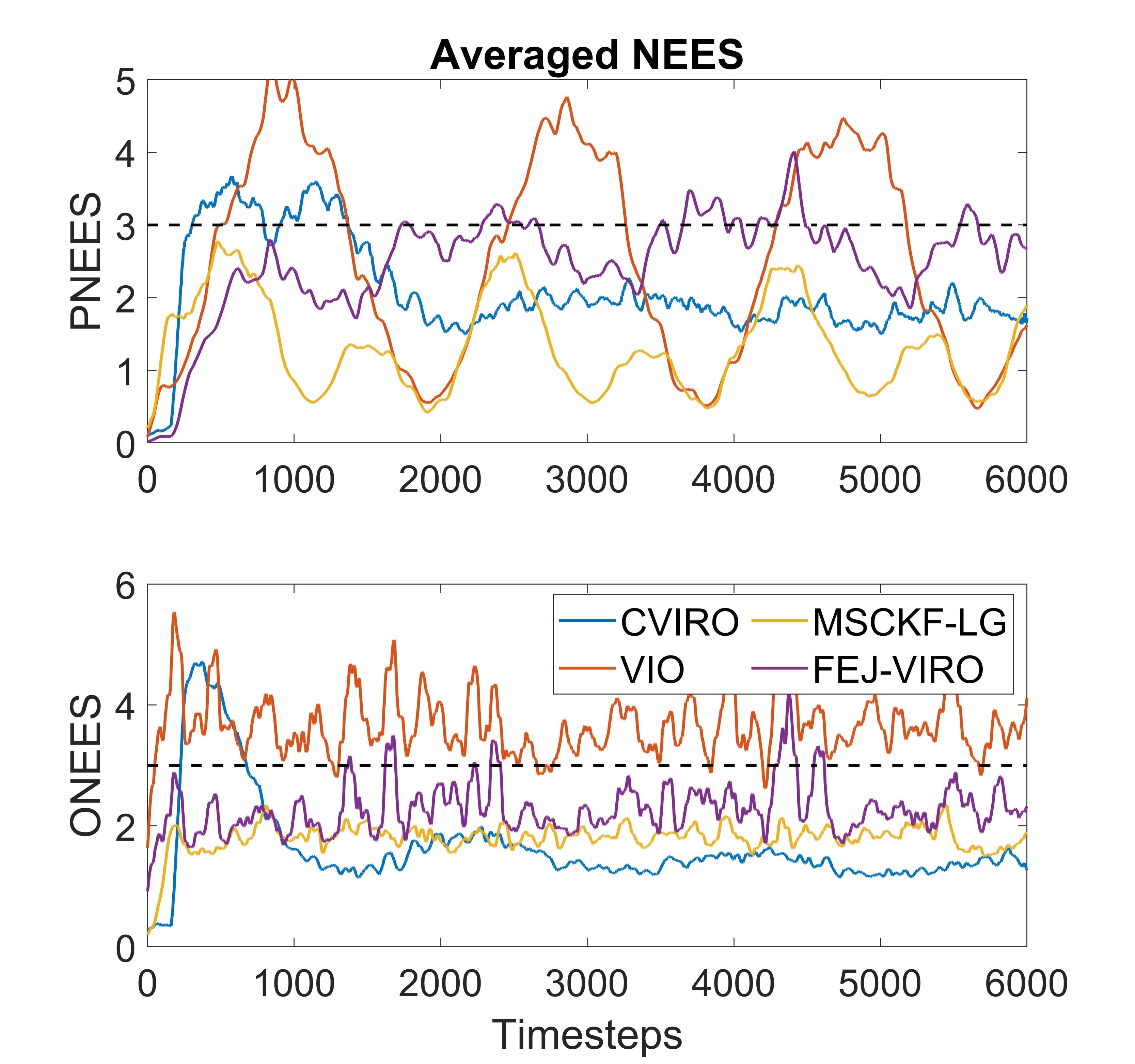}
        \caption{NEES of traj. a}
        \label{fig:subfig2}
    \end{subfigure}

    \begin{subfigure}[b]{0.23\textwidth}
        \centering
        \includegraphics[width=\textwidth]{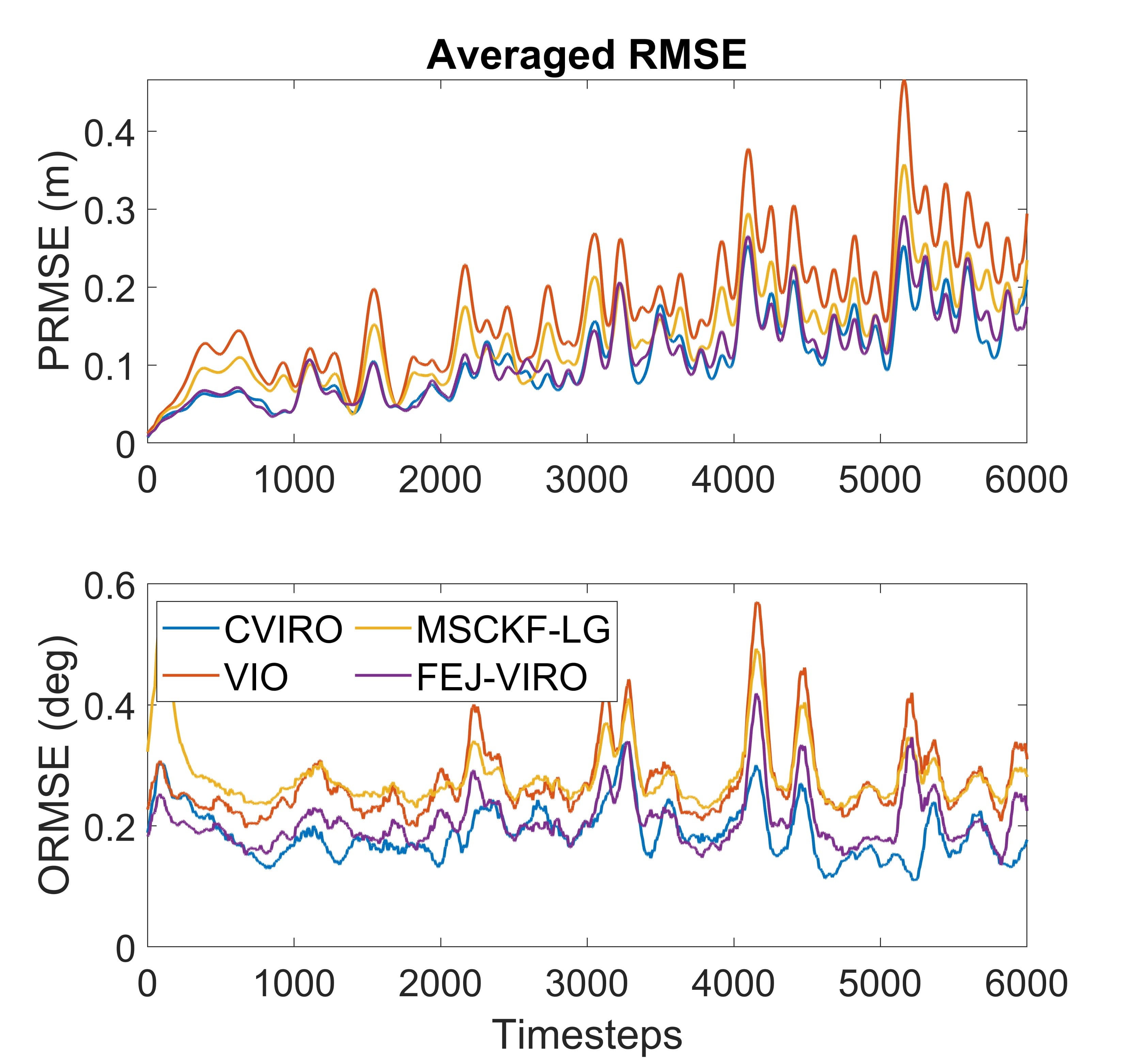}
        \caption{RMSE of traj. b}
        \label{fig:subfig3}
    \end{subfigure}
    \begin{subfigure}[b]{0.23\textwidth}
        \centering
        \includegraphics[width=\textwidth]{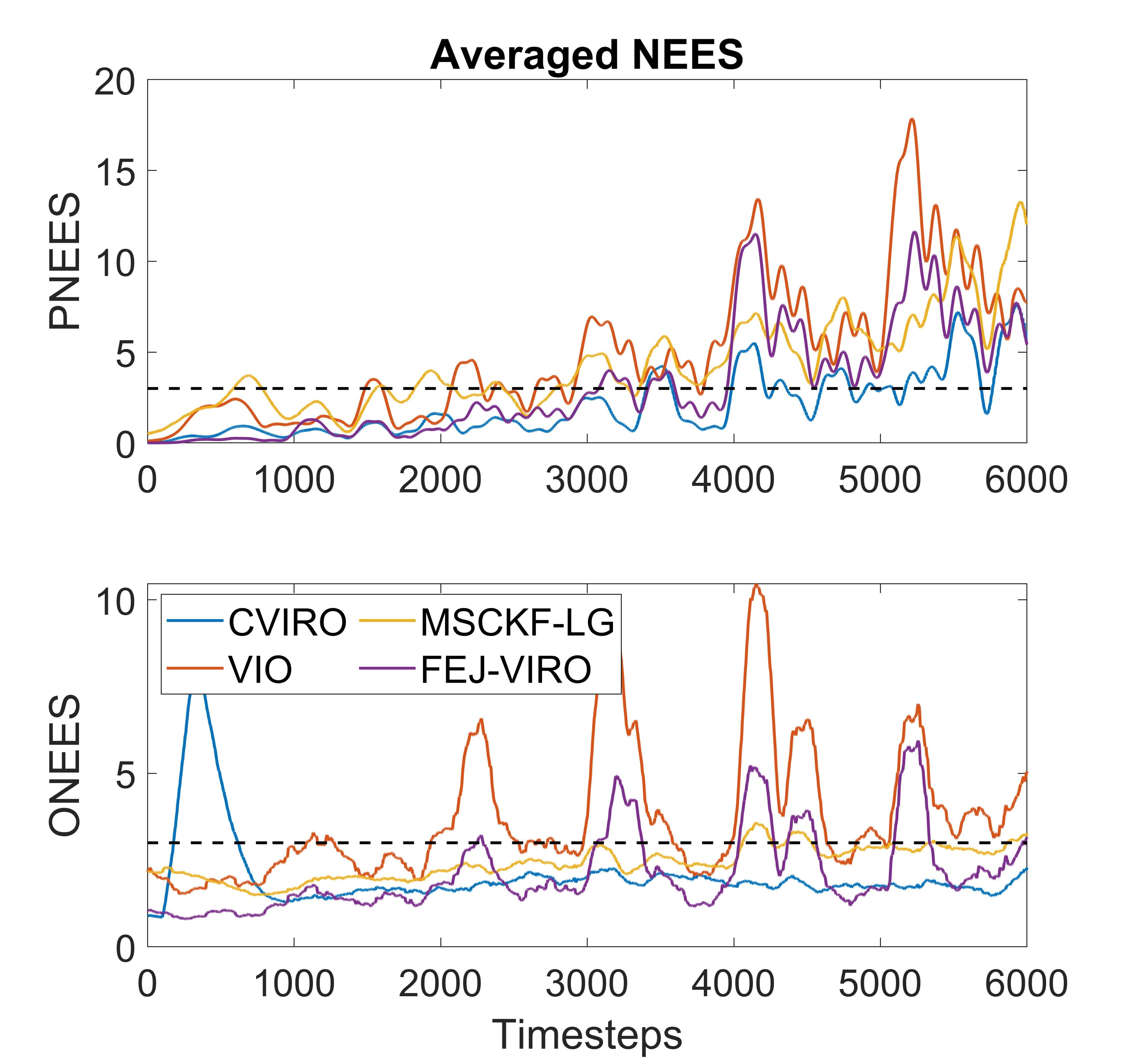}
        \caption{NEES of traj. b}
        \label{fig:subfig4}
    \end{subfigure}
    
    \begin{subfigure}[b]{0.23\textwidth}
        \centering
        \includegraphics[width=\textwidth]{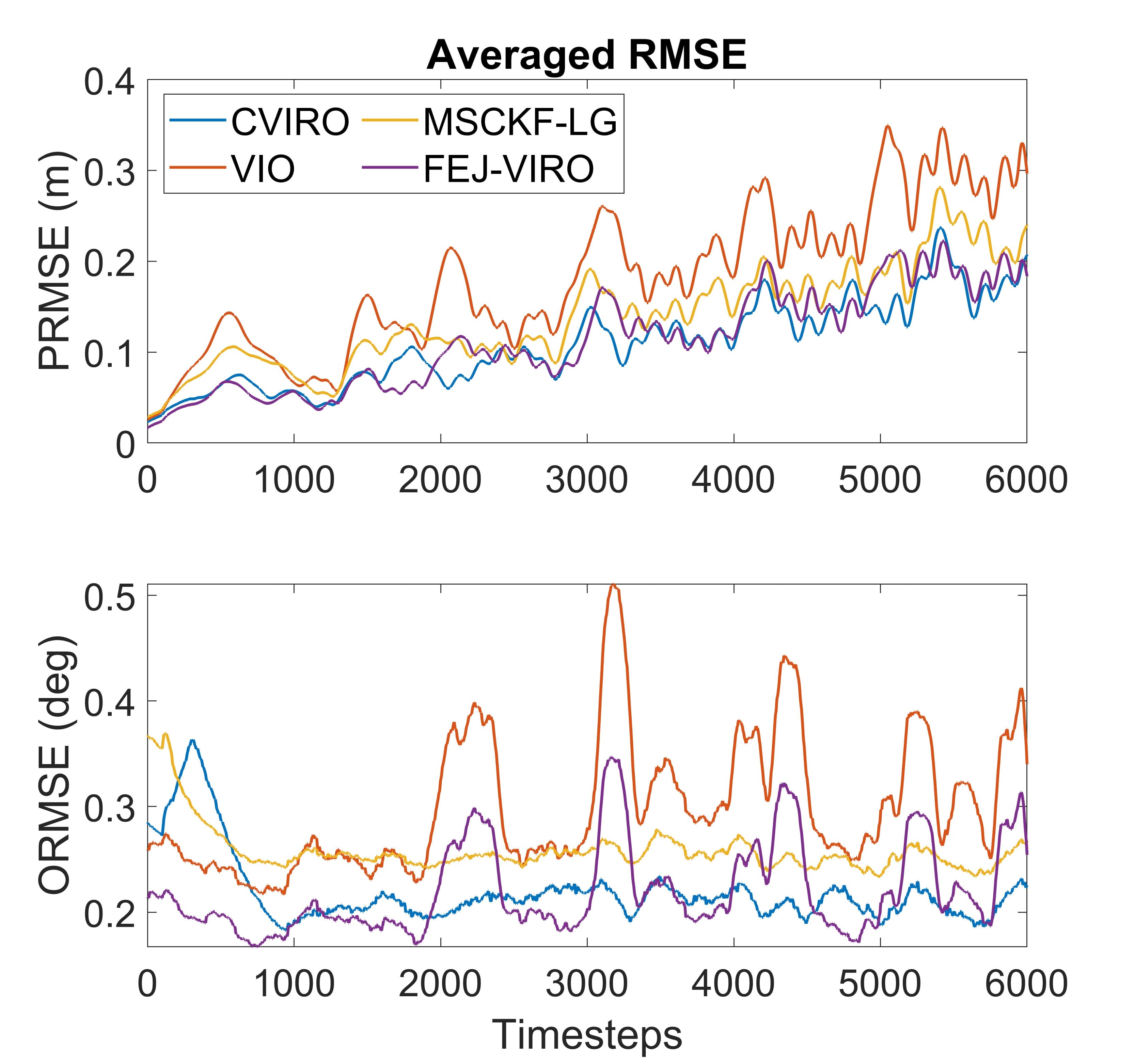}
        \caption{RMSE of traj. c}
        \label{fig:subfig5}
    \end{subfigure}
    \begin{subfigure}[b]{0.23\textwidth}
        \centering
        \includegraphics[width=\textwidth]{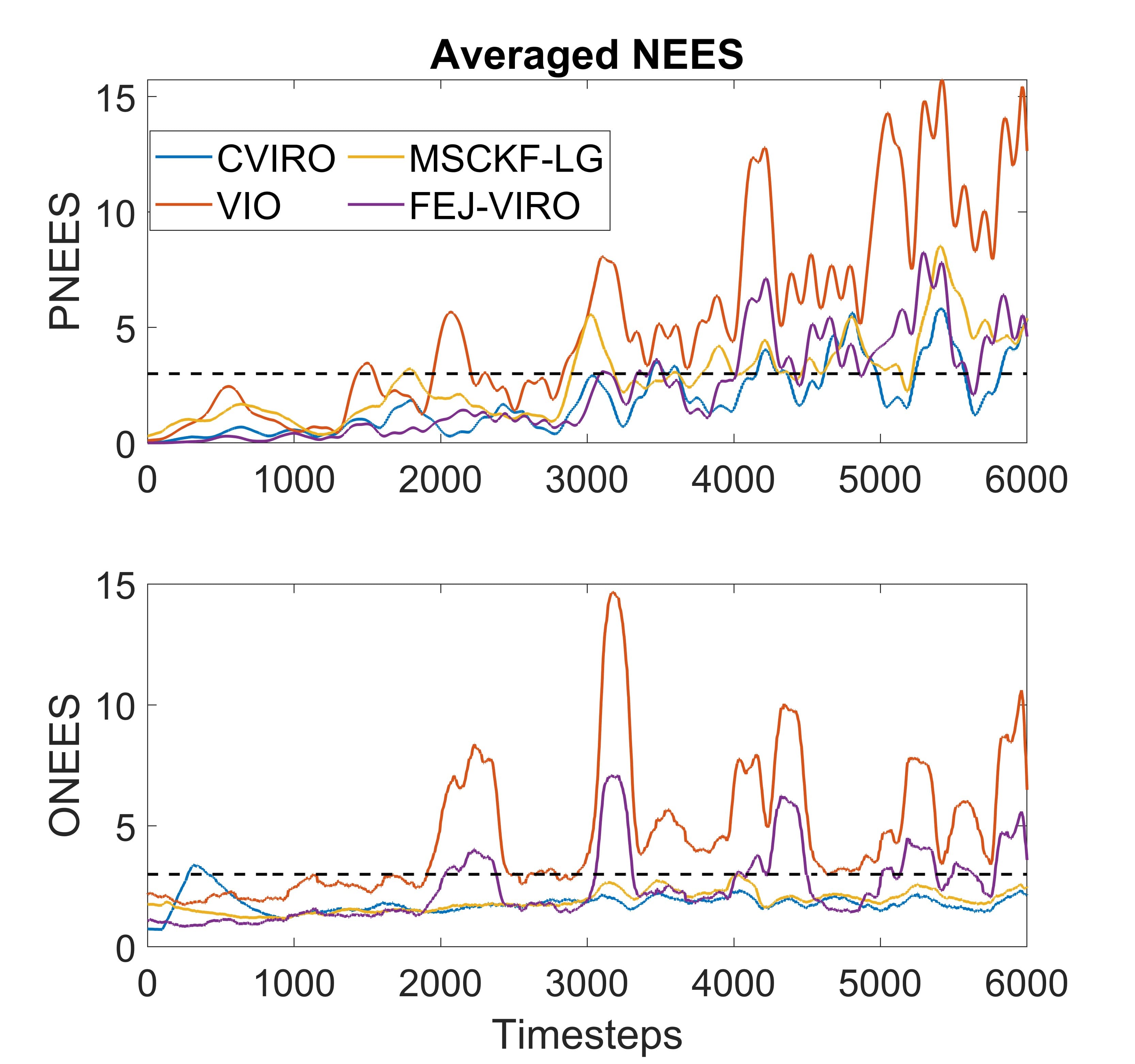}
        \caption{NEES of traj. c}
        \label{fig:subfig6}
    \end{subfigure}
    \caption{Averaged RMSE and NEES results of 50 Monte-Carlo simulation trails across all the simulated trajectories. For NEES, if the curve exceeds the dotted line at 3, it indicates that the algorithm is inconsistent. As observed, our algorithm remains below 3 in most cases, demonstrating a consistency advantage over the other methods.}
    \label{fig_traj}
\end{figure}

\begin{table}[t]
\vspace{-1ex}
\centering
\caption{Averaged RMSE and NEES errors over 50 Monte-Carlo runs.}
\begin{tabular}{*{6}{c}}
  \toprule[1.2pt]
  & Method & PRMSE & ORMSE& PNEES & ONEES\\
  \midrule[1.2pt]
  \multirow{4}*{Traj. a}
  &\textbf{CVIRO} &  0.027&  0.199& 2.456& 2.083\\
  &{FEJ-VIRO} &  0.029&  0.203& 2.635& 2.321\\
  & VIO &  0.054 &  0.288& 2.902& 3.815\\ 
  & MSCKF-LG &  0.040 &  0.268& 1.902& 2.215\\
  \midrule
  \multirow{4}*{Traj. b}
  &\textbf{CVIRO} &  0.110&  0.260& 2.872& 2.331\\
  &{FEJ-VIRO} &  0.145 &  0.262& 3.456& 2.677\\
  & VIO &  0.261 &  0.285& 4.102& 3.967\\ 
  & MSCKF-LG &  0.202 &  0.283& 3.105& 2.622\\
  \midrule
  \multirow{4}*{Traj. c}
  &\textbf{CVIRO} &  0.104&  0.245& 2.623& 2.159\\
  &{FEJ-VIRO} &  0.139 &  0.255 & 2.658 & 2.710\\
  & VIO &  0.273 &  0.255& 4.879 & 3.609\\ 
  & MSCKF-LG &  0.230 &  0.250 & 2.561& 2.265\\
  \bottomrule[1.2pt]
  \end{tabular}
\label{Tab_sim}
\end{table}

\subsection{Real-world experiments}
We further implement the proposed CVIRO algorithm based on OpenVINS \cite{OVS}, an open-source filter-based visual-inertial estimator, and evaluate its performance using the VIRAL dataset \cite{ntuviral}. The VIRAL dataset provides multi-modal sensor data collected from a drone, including IMU, stereo camera, lidar, and UWB, enabling comprehensive validation of the algorithm under realistic conditions. Specifically, the robot is equipped with four UWB tags that measure range data to three UWB anchor nodes for localization. Each tag has been pre-calibrated, ensuring that the terms ${^I\mathbf p_T}$ and $\mathbf b_{r_k}$ in \eqref{eq_uwb} are accurately known. Moreover, since the VIRAL dataset already provides compatibility with OpenVINS, it enables rapid development and testing of our proposed algorithm. To evaluate the estimation results, we follow the "Evaluation Recommendation" section of their tutorial \cite{ntuviral} and implement the result evaluation tools based on their provided MATLAB scripts. We present the estimation results of the proposed algorithm for different trajectories in Figure \ref{Fig_vr_traj}. The results demonstrate the effectiveness of the our CVIRO algorithm since the estimated trajectories are very close to the groundtruth trajectories.

We compare the proposed CVIRO method with the original OpenVINS and FEJ-VIRO to evaluate whether the proposed approach can effectively calibrate and fuse UWB measurements to enhance localization accuracy. All comparison results across different trajectories for each algorithm are provided in Table \ref{tab:vio_comparison}. For the FEJ-VIRO results, we directly use the data reported in their paper. From this table, we can see that the proposed CVIRO algorithm significantly outperforms the original OV algorithm by effectively integrating UWB ranging measurements, leading to improved localization accuracy. Moreover, the Lie group-based representation of our algorithm ensures observability consistency by leveraging the property of the invariant error, which further enhances the estimator's performance. In addition, our algorithm achieves comparable results to the FEJ-VIRO method across all tested sequences. Considering that our algorithm only utilizes MSCKF features and does not incorporate SLAM features like FEJ-VIRO in the current stage, this result further demonstrates the effectiveness of our approach. Since incorporating SLAM features into the Lie group-based state representation introduces additional complexity, it is currently beyond the scope of our work. However, from another perspective, we will address this limitation in future work and integrate SLAM features to further enhance localization accuracy.

\begin{figure}[htbp]
    \centering
    \begin{subfigure}[b]{0.23\textwidth}
        \centering
        \includegraphics[width=\textwidth]{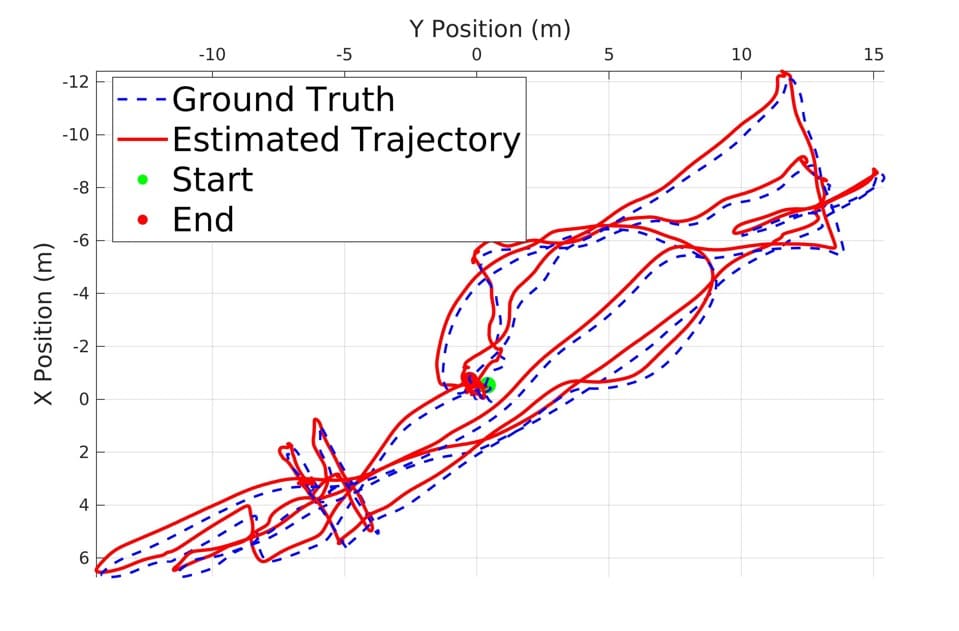}
        \caption{eee\_01}
        \label{fig:eee_01}
    \end{subfigure}
    \begin{subfigure}[b]{0.23\textwidth}
        \centering
        \includegraphics[width=\textwidth]{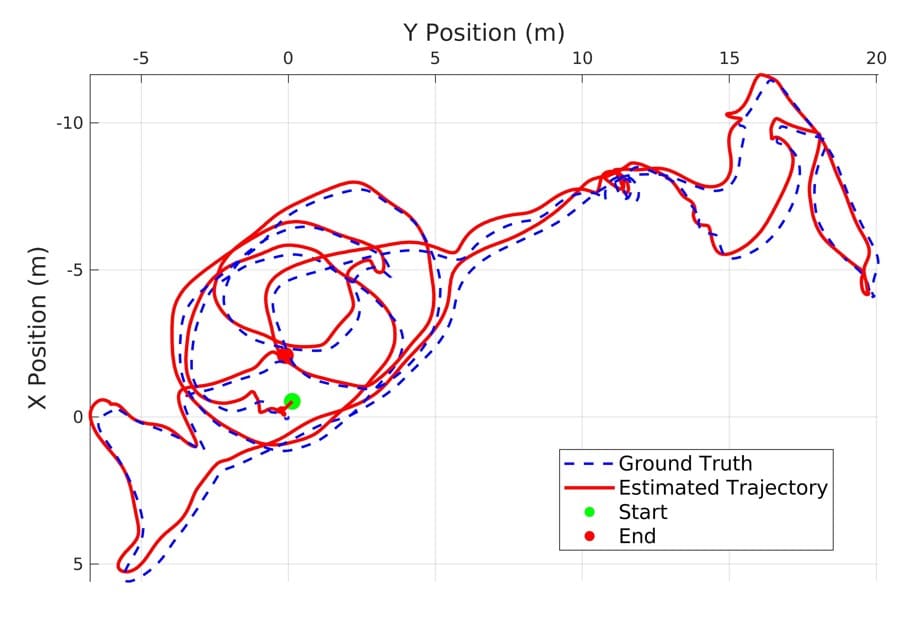}
        \caption{eee\_02}
        \label{fig:eee_02}
    \end{subfigure}
    \begin{subfigure}[b]{0.23\textwidth}
        \centering
        \includegraphics[width=\textwidth]{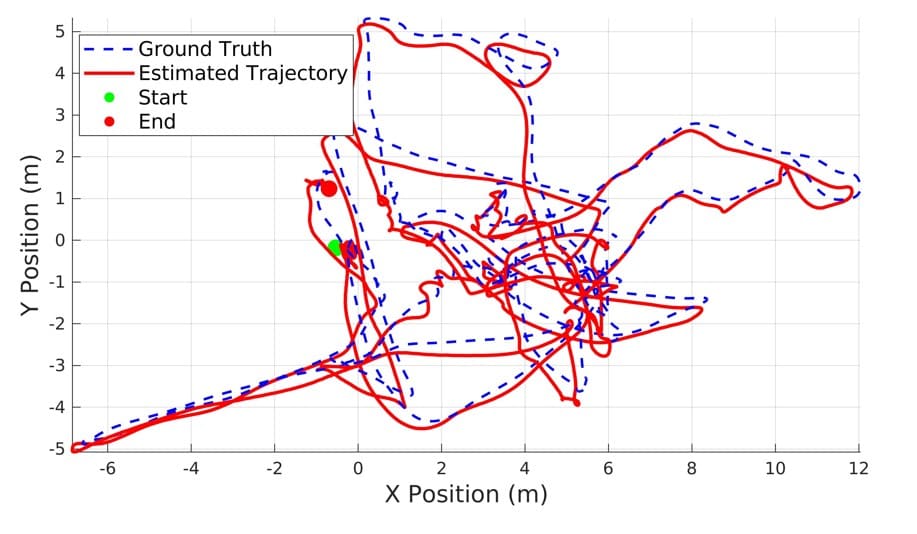}
        \caption{nya\_01}
        \label{fig:eee_01}
    \end{subfigure}
    \begin{subfigure}[b]{0.23\textwidth}
        \centering
        \includegraphics[width=\textwidth]{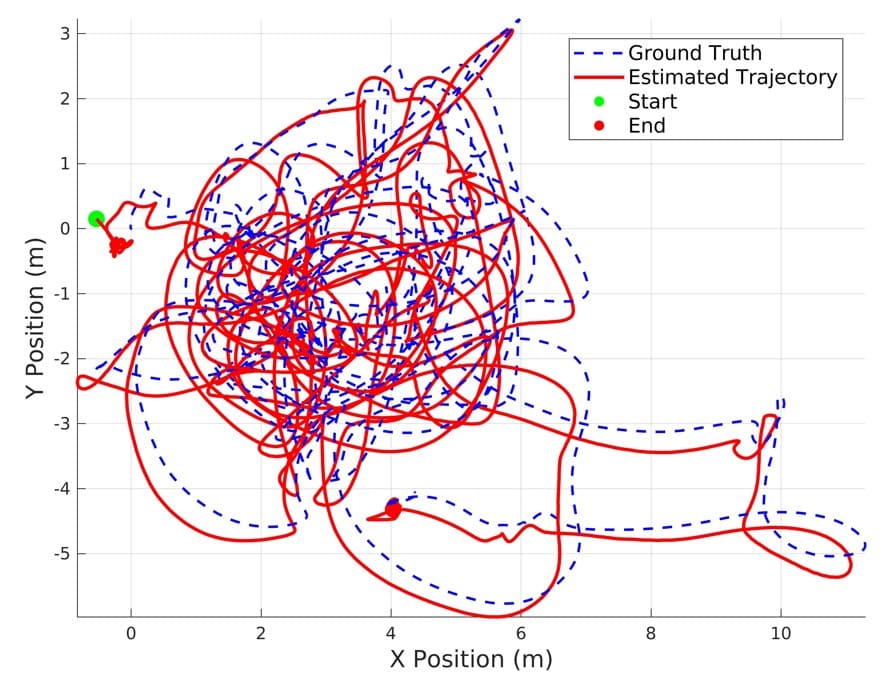}
        \caption{nya\_02}
        \label{fig:eee_02}
    \end{subfigure}
    \caption{Estimated trajectories for different sequences of the VIRAL dataset.}
    \label{Fig_vr_traj}
\end{figure}

\begin{table}[h]
    \centering
    \renewcommand{\arraystretch}{1.2}
    \caption{Comparison ATE results across different algorithms on VIRAL dataset. All the algorithms use $11$ clones during estimation. OV refers to OpenVINS, and F-VIRO refers to FEJ-VIRO method.}
    \begin{tabular}{lcccccc}
        \toprule
        \textbf{Method} & \textbf{eee\_01} & \textbf{eee\_02} & \textbf{eee\_03} & \textbf{nya\_01} & \textbf{nya\_02} & \textbf{nya\_03} \\
        \midrule
        F-VIRO  & 0.479 & 0.340 & 0.348 & 0.505 & \textbf{0.202} & 0.290 \\
        OV     & 0.645 & 0.605 & 0.417 & 0.739 & 0.473 & 0.690 \\
        CVIRO  & \textbf{0.424} & \textbf{0.316} & \textbf{0.322} & \textbf{0.490} & 0.265 & \textbf{0.267} \\
        \bottomrule
    \end{tabular}
    \label{tab:vio_comparison}
\end{table}

\section{Conclusion}
In this paper, we proposed a consistent and tightly-coupled visual-inertial-ranging odometry (CVIRO) system based on the Lie group to improve the accuracy and consistency of UWB-aided VIO. Our approach explicitly integrates UWB anchor states into the system state, enabling autonomous calibration while accounting for calibration uncertainty, which enhances estimation robustness and consistency. To ensure observability consistency, we introduced a Lie group-based multi-state constrained Kalman filter (MSCKF) that leverages right-invariant errors. We analytically proved that the proposed CVIRO estimator preserves the system’s four unobservable directions, ensuring that the estimator maintains the correct system structure without requiring external Jacobian manipulation. Through extensive Monte Carlo simulations and real-world experiments, we demonstrated the effectiveness of our work, which highlights the advantages of Lie group representations in tightly coupled UWB-aided VIO systems, offering a promising solution for drift-free, real-time state estimation in robotic applications. In future work, we aim to incorporate SLAM features and extend our framework to multi-robot scenarios.

\bibliographystyle{IEEEtran}
\bibliography{main} 

\begin{thebibliography}{10}
\providecommand{\url}[1]{#1}
\csname url@samestyle\endcsname
\providecommand{\newblock}{\relax}
\providecommand{\bibinfo}[2]{#2}
\providecommand{\BIBentrySTDinterwordspacing}{\spaceskip=0pt\relax}
\providecommand{\BIBentryALTinterwordstretchfactor}{4}
\providecommand{\BIBentryALTinterwordspacing}{\spaceskip=\fontdimen2\font plus
\BIBentryALTinterwordstretchfactor\fontdimen3\font minus \fontdimen4\font\relax}
\providecommand{\BIBforeignlanguage}[2]{{%
\expandafter\ifx\csname l@#1\endcsname\relax
\typeout{** WARNING: IEEEtran.bst: No hyphenation pattern has been}%
\typeout{** loaded for the language `#1'. Using the pattern for}%
\typeout{** the default language instead.}%
\else
\language=\csname l@#1\endcsname
\fi
#2}}
\providecommand{\BIBdecl}{\relax}
\BIBdecl

\bibitem{OVS}
P.~Geneva, K.~Eckenhoff, W.~Lee, Y.~Yang, and G.~Huang, ``Openvins: A research platform for visual-inertial estimation,'' in \emph{2020 IEEE International Conference on Robotics and Automation (ICRA)}, 2020, pp. 4666--4672.

\bibitem{VM}
T.~Qin, P.~Li, and S.~Shen, ``Vins-mono: A robust and versatile monocular visual-inertial state estimator,'' \emph{IEEE Transactions on Robotics}, vol.~34, no.~4, pp. 1004--1020, 2018.

\bibitem{Jia2022}
S.~Jia, Y.~Jiao, Z.~Zhang, R.~Xiong, and Y.~Wang, ``Fej-viro: A consistent first-estimate jacobian visual-inertial-ranging odometry,'' in \emph{2022 IEEE/RSJ International Conference on Intelligent Robots and Systems (IROS)}, 2022, pp. 1336--1343.

\bibitem{Hc2024}
C.~Hu, P.~Huang, and W.~Wang, ``Tightly coupled visual-inertial-uwb indoor localization system with multiple position-unknown anchors,'' \emph{IEEE Robotics and Automation Letters}, vol.~9, no.~1, pp. 351--358, 2024.

\bibitem{NTM2021}
T.-M. Nguyen, M.~Cao, S.~Yuan, Y.~Lyu, T.~H. Nguyen, and L.~Xie, ``Liro: Tightly coupled lidar-inertia-ranging odometry,'' in \emph{2021 IEEE International Conference on Robotics and Automation (ICRA)}, 2021, pp. 14\,484--14\,490.

\bibitem{HG2014}
G.~Huang, M.~Kaess, and J.~J. Leonard, ``Towards consistent visual-inertial navigation,'' in \emph{2014 IEEE International Conference on Robotics and Automation (ICRA)}, 2014, pp. 4926--4933.

\bibitem{FEJ}
G.~P. Huang, A.~I. Mourikis, and S.~I. Roumeliotis, ``A first-estimates jacobian ekf for improving slam consistency,'' in \emph{Experimental Robotics}, O.~Khatib, V.~Kumar, and G.~J. Pappas, Eds.\hskip 1em plus 0.5em minus 0.4em\relax Berlin, Heidelberg: Springer Berlin Heidelberg, 2009, pp. 373--382.

\bibitem{FEJ2}
C.~Chen, Y.~Yang, P.~Geneva, and G.~Huang, ``Fej2: A consistent visual-inertial state estimator design,'' in \emph{2022 International Conference on Robotics and Automation (ICRA)}, 2022, pp. 9506--9512.

\bibitem{zzSPLY2025}
\BIBentryALTinterwordspacing
Y.~Zhou, ``Supplementary materials: A consistent and tightly-coupled visual-inertial-ranging odometry on lie groups,'' 2025. [Online]. Available: \url{https://mason.gmu.edu/~xwang64/papers/cviro_supp.pdf}
\BIBentrySTDinterwordspacing

\bibitem{WC2017}
C.~Wang, H.~Zhang, T.-M. Nguyen, and L.~Xie, ``Ultra-wideband aided fast localization and mapping system,'' in \emph{2017 IEEE/RSJ International Conference on Intelligent Robots and Systems (IROS)}, 2017, pp. 1602--1609.

\bibitem{PF2017}
F.~J. Perez-Grau, F.~Caballero, L.~Merino, and A.~Viguria, ``Multi-modal mapping and localization of unmanned aerial robots based on ultra-wideband and rgb-d sensing,'' in \emph{2017 IEEE/RSJ International Conference on Intelligent Robots and Systems (IROS)}, 2017, pp. 3495--3502.

\bibitem{ZJ2022}
J.-R. Zhan and H.-Y. Lin, ``Improving visual inertial odometry with uwb positioning for uav indoor navigation,'' in \emph{2022 26th International Conference on Pattern Recognition (ICPR)}, 2022, pp. 4189--4195.

\bibitem{YB2021}
B.~Yang, J.~Li, and H.~Zhang, ``Uvip: Robust uwb aided visual-inertial positioning system for complex indoor environments,'' in \emph{2021 IEEE International Conference on Robotics and Automation (ICRA)}, 2021, pp. 5454--5460.

\bibitem{SSL2021}
S.~Shin, E.~Lee, J.~Choi, and H.~Myung, ``Mir-vio:mutual information residual-based visual inertial odometry with uwb fusion for robust localization,'' in \emph{2021 21st International Conference on Control, Automation and Systems (ICCAS)}, 2021, pp. 91--96.

\bibitem{ZS20222}
S.~Zheng, Z.~Li, Y.~Liu, H.~Zhang, P.~Zheng, X.~Liang, Y.~Li, X.~Bu, and X.~Zou, ``Uwb-vio fusion for accurate and robust relative localization of round robotic teams,'' \emph{IEEE Robotics and Automation Letters}, vol.~7, no.~4, pp. 11\,950--11\,957, 2022.

\bibitem{TJ2018}
J.~Tiemann, A.~Ramsey, and C.~Wietfeld, ``Enhanced uav indoor navigation through slam-augmented uwb localization,'' in \emph{2018 IEEE International Conference on Communications Workshops (ICC Workshops)}, 2018, pp. 1--6.

\bibitem{NT2020}
T.~H. Nguyen, T.-M. Nguyen, and L.~Xie, ``Tightly-coupled single-anchor ultra-wideband-aided monocular visual odometry system,'' in \emph{2020 IEEE International Conference on Robotics and Automation (ICRA)}, 2020, pp. 665--671.

\bibitem{Cao2021}
Y.~Cao and G.~Beltrame, ``{VIR-SLAM: Visual, Inertial, and Ranging SLAM for Single and Multi-Robot Systems},'' \emph{Autonomous Robots}, vol.~45, pp. 905--917, Sep. 2021.

\bibitem{DGS2023}
G.~Delama, F.~Shamsfakhr, S.~Weiss, D.~Fontanelli, and A.~Fomasier, ``Uvio: An uwb-aided visual-inertial odometry framework with bias-compensated anchors initialization,'' in \emph{2023 IEEE/RSJ International Conference on Intelligent Robots and Systems (IROS)}, 2023, pp. 7111--7118.

\bibitem{YYL2022}
Y.~Yang, C.~Chen, W.~Lee, and G.~Huang, ``Decoupled right invariant error states for consistent visual-inertial navigation,'' \emph{IEEE Robotics and Automation Letters}, vol.~7, no.~2, pp. 1627--1634, 2022.

\bibitem{Xuz2023}
J.~Xu, P.~Zhu, Y.~Zhou, and W.~Ren, ``Distributed invariant extended kalman filter using lie groups: Algorithm and experiments,'' \emph{IEEE Transactions on Control Systems Technology}, vol.~31, no.~6, pp. 2777--2789, 2023.

\bibitem{BA2017}
A.~Barrau and S.~Bonnabel, ``The invariant extended kalman filter as a stable observer,'' \emph{IEEE Transactions on Automatic Control}, vol.~62, no.~4, pp. 1797--1812, 2017.

\bibitem{Heo2018}
S.~Heo and C.~G. Park, ``Consistent ekf-based visual-inertial odometry on matrix lie group,'' \emph{IEEE Sensors Journal}, vol.~18, no.~9, pp. 3780--3788, 2018.

\bibitem{BM2018}
M.~Brossard, S.~Bonnabel, and A.~Barrau, ``Invariant kalman filtering for visual inertial slam,'' in \emph{2018 21st International Conference on Information Fusion (FUSION)}, 2018, pp. 2021--2028.

\bibitem{RH2020}
R.~Hartley, M.~Ghaffari, R.~M. Eustice, and J.~W. Grizzle, ``Contact-aided invariant extended kalman filtering for robot state estimation,'' \emph{The International Journal of Robotics Research}, vol.~39, no.~4, pp. 402--430, 2020.

\bibitem{zhang2017convergence}
T.~Zhang, K.~Wu, J.~Song, S.~Huang, and G.~Dissanayake, ``Convergence and consistency analysis for a 3-d invariant-ekf slam,'' \emph{IEEE Robotics and Automation Letters}, vol.~2, no.~2, pp. 733--740, 2017.

\bibitem{MSCKF}
A.~I. Mourikis and S.~I. Roumeliotis, ``A multi-state constraint kalman filter for vision-aided inertial navigation,'' in \emph{Proceedings 2007 IEEE International Conference on Robotics and Automation}, 2007, pp. 3565--3572.

\bibitem{ntuviral}
T.-M. Nguyen, S.~Yuan, M.~Cao, Y.~Lyu, T.~H. Nguyen, and L.~Xie, ``Ntu viral: A visual-inertial-ranging-lidar dataset, from an aerial vehicle viewpoint,'' \emph{The International Journal of Robotics Research}, vol.~41, no.~3, pp. 270--280, 2022.

\bibitem{LIVIO2015}
M.~Li and A.~I. Mourikis, ``High-precision, consistent ekf-based visual-inertial odometry,'' \emph{The International Journal of Robotics Research}, vol.~32, no.~6, pp. 690--711, 2013.

\end{thebibliography}

\addtolength{\textheight}{-12cm}   
\end{document}